\documentclass[journal]{IEEEtran}

\usepackage{subfigure}
\usepackage{bm}
\usepackage{amssymb}
\usepackage{epsfig}
\usepackage{amsmath}
\usepackage{amsthm}
\usepackage{xcolor}
\usepackage{graphicx}
\usepackage{epstopdf}
\usepackage{amsfonts}%
\usepackage{indentfirst}
\usepackage{color}
\usepackage[ruled]{algorithm2e}
\usepackage{floatrow}
\usepackage{comment}

\newtheorem{problem}{\textbf{Problem}}
\newtheorem{definition}{\textbf{Definition}}

\newtheorem{theorem}{\rm\textbf{Theorem}}

\newtheorem{remark}{\rm\textbf{Remark}}
\newtheorem{example}{\rm\textbf{Example}}

\begin{document}
%
\title{Rule-based Evaluation and Optimal Control for Autonomous Driving}
%
%
%

\author{Wei Xiao,~\IEEEmembership{Member,~IEEE}, Noushin Mehdipour, ~\IEEEmembership{Member,~IEEE}, Anne Collin, Amitai Y. Bin-Nun, Emilio Frazzoli,~\IEEEmembership{Fellow, IEEE}, Radboud Duintjer Tebbens, Calin Belta~\IEEEmembership{Fellow, IEEE}

\thanks{Wei Xiao is with Boston University {\tt xiaowei@bu.edu}. This work was performed while he was an intern at Motional.}
	\thanks{Noushin Mehdipour, Anne Collin, Amitai Y. Bin-Nun, and Radboud Duintjer Tebbens are with Motional {\tt \{noushin.mehdipour, anne.collin, amitai.binnun, radboud.tebbens\}@motional.com}
	}
	\thanks{Emilio Frazzoli is with Motional and with ETH Zurich {\tt emilio.frazzoli@motional.com}
	}
	\thanks{Calin Belta is with Boston University and with  Motional {\tt calin.belta@motional.com}
	}
	\thanks{Some results from this paper appeared in the Proceedings of the 2021 ACM/IEEE International Conference on Cyber-Physical Systems (ICCPS) \cite{Xiao2021ICCPS}.}
}

\maketitle


\vspace{2mm}

\begin{abstract}
We develop optimal control strategies for autonomous vehicles (AVs) that are required to meet complex specifications imposed as rules of the road (ROTR) and locally specific cultural  expectations of reasonable driving behavior. We formulate these specifications as rules, and specify their priorities by constructing a priority structure,  called \underline{T}otal \underline{OR}der over e\underline{Q}uivalence classes (TORQ). We propose a recursive framework, in which the satisfaction of the rules in the priority structure are iteratively relaxed in reverse order of priority. 

Central to this framework is an optimal control problem, where convergence to desired states is achieved using Control Lyapunov Functions (CLFs) and clearance with other road users is enforced through Control Barrier Functions (CBFs). 
We present offline and online approaches to this problem. In the latter, the AV has limited sensing range that affects the activation of the rules, and the control is generated using a receding horizon (Model Predictive Control, MPC) approach.  We also show how the offline method can be used for after-the-fact (offline) pass/fail evaluation of trajectories - a given trajectory is rejected if we can find a controller producing a trajectory that leads to less violation of the rule priority structure. We present case studies with multiple driving scenarios to demonstrate the effectiveness of the algorithms, and to compare the offline and online versions of our proposed framework. 
\end{abstract}


%
\IEEEpeerreviewmaketitle

\section{INTRODUCTION}

\label{sec:intro}

With the development and integration of cyber physical and safety critical systems in various engineering disciplines, there is an increasing need for computational tools for verification and control of such systems according to rich and complex specifications. A prominent example is autonomous driving, which received a lot of attention during the last decade. 
Besides common objectives in optimal control problems, such as minimizing the energy consumption and travel time, and constraints on control variables, such as maximum acceleration, autonomous vehicles (AVs) should follow complex rules of the road (ROTR) with different priorities. 
They should also meet cultural expectations of reasonable driving behavior \cite{nolte2017towards,shalev2017formal,parseh2019pre,ulbrich2013probabilistic,qian2014priority,iso2019pas,Collin2020}. 
For example, an AV should avoid collisions with other road users (higher priority), maintain longitudinal clearance with the lead car (lower priority), and drive faster than some minimum speed limit (still lower priority).
Inspired by \cite{Censi2020}, we formulate these behavior specifications as a set of rules with a priority structure that captures their importance \cite{Censi2020}. 

To accommodate the rules and the priority structure, we formulate the AV control problem as an optimal control problem, in which the satisfaction of the rules and some vehicle limitations are enforced by Control Barrier Functions (CBF) \cite{Ames2017}. To minimize a global notion of rule violation, we formulate iterative rule relaxations according to their priorities. We propose offline and online versions of the control problem, and in both cases require the AV to follow a reference trajectory. In the offline version, trajectory tracking is achieved through additional constraints implemented using Control Lyapunov Functions (CLF)\cite{Freeman1996}. In online control, we track the reference trajectory by including the tracking error in the cost, and by performing optimization over a receding horizon. 

Control Lyapunov functions \cite{Freeman1996,Artstein1983} have been used to stabilize systems to desired states. CBFs enforce set forward-invariance \cite{Tee2009,Wisniewski2013}, and have been adopted to guarantee the satisfaction of safety specifications \cite{Ames2017,wang2017safety,Lindemann2018}. In \cite{Ames2017,Glotfelter2017}, the constraints induced by CBFs and CLFs were used to formulate quadratic programs (QPs) that could be solved in real time to stabilize affine control systems while optimizing quadratic costs and satisfying state and control constraints. The main limitation of this approach is that the resulting QPs can easily become infeasible, especially when bounds on control inputs are imposed in addition to the safety specifications and the state constraints, or for constraints with high relative degree \cite{Xiao2019}. Relaxations of the (hard) CLF \cite{Aaron2012,Ames2017}
and CBF \cite{Xiao2019} constraints 
have been proposed to address this issue. 

The approaches described above do not consider the (relative) importance of the constraints during their relaxations. With particular relevance to the application considered here, AVs often deal with situations where they cannot satisfy all the ROTR or expectations of reasonable driving behavior. For instance, consider a scenario where a pedestrian walks to the lane in which the AV is driving - it is impossible for the AV to avoid a collision with the pedestrian or another vehicles, stay in lane, and drive faster than the minimum speed limit at the same time. 
Given the relative priorities of these driving rules, a reasonable AV behavior would be to avoid a collision with the pedestrian or other vehicles (high priority), even if doing so would violate low or medium priority rules, e.g., by reducing speed to a value lower than the minimum speed limit, and/or changing lane if the adjacent lane is free. The maximum satisfaction and minimum violation of a set of rules expressed in temporal logic were studied in \cite{dimitrova2018maximum,tuumova2013minimum} and solved by assigning positive numerical weights to formulas based on their priorities \cite{tuumova2013minimum,wSTL}. In \cite{Censi2020}, the authors proposed \emph{rulebooks}, a framework in which relative priorities were captured by a pre-order. In conjunction with rule violation scores, rulebooks were used to rank vehicle trajectories. These works do not take the vehicle dynamics into account, or assume very simple forms, such as finite transition systems. The violation scores are example - specific, or are simply the quantitative semantics of the logic used to formulate the rules. In their current form, they capture worst case scenarios and are non-differentiable, and   cannot be used for generating controllers for realistic vehicle dynamics.

In this paper, we draw inspiration from Signal Temporal Logic (STL) \cite{Maler2004} to formalize ROTR and other driving rules and to quantify the degree of violation of the rules. We build on the rulebooks from \cite{Censi2020} to construct a specific type of  priority structure, called \underline{T}otal \underline{OR}der over e\underline{Q}uivalence classes (TORQ). The main contribution of this paper is the development of offline and online iterative procedures that use TORQ to determine control strategies that minimize rule violation globally.
The online version can activate and disactivate rules depending on local sensing of relevant traffic participants or features (e.g., parked cars, pedestrians, road dividers). 
We show how the offline procedure can be adapted to develop transparent and reproducible rule-based pass/fail evaluation of AV trajectories in test scenarios. 
Central to these approaches is an optimization problem based on \cite{Xiao2019}, which uses detailed vehicle dynamics, ensures the satisfaction of ``hard" vehicle limitations (e.g., acceleration constraints), and can accommodate rule constraints with high relative degree. Another key contribution of this work is the formal definition of a speed dependent, optimal over-approximation of a vehicle footprint that ensures differentiability of clearance-type rules, which enables the use of powerful approaches based on CBF and CLF. Finally, we test the proposed framework in various traffic scenarios using a novel user-friendly software tool.


\section{PRELIMINARIES}
\label{sec:pre}

\subsection{Vehicle Dynamics}
\label{sec:vd}

Consider an affine control system given by:\vspace{-3pt}
\begin{equation}\label{eqn:affine}
\dot{\bm{x}}=f(\bm x)+g(\bm x)\bm u, 
\vspace{-3pt}
\end{equation}
where $\bm x\in X\subset\mathbb{R}^{n}$ ($X$ is the state constraint set), $\dot{()}$ denotes differentiation with respect to time, 
$f:\mathbb{R}^{n}\rightarrow\mathbb{R}^{n}$
and $g:\mathbb{R}^{n}\rightarrow\mathbb{R}^{n\times q}$ are globally
Lipschitz, and $\bm u\in U\subset\mathbb{R}^{q}$, where $U$ is the control constraint set
defined as:
\begin{equation}
U:=\{\bm u\in\mathbb{R}^{q}:\bm u_{min}\leq\bm u\leq\bm u_{max}\},
\label{eqn:control}%
\end{equation}
with $\bm u_{min},\bm u_{max}\in\mathbb{R}^{q}$, and the inequalities are
interpreted componentwise. We use $\bm{x}(t)$ to refer to a trajectory of (\ref{eqn:affine}) at a specific time $t$, and we use $\mathcal{X}$ to denote a whole trajectory starting at time 0 and ending at a final time specified by a scenario.  
Note that most vehicle dynamics, such as traditional dynamics defined with respect to an inertial frame \cite{Ames2017} and dynamics defined along a given reference trajectory \cite{Rucco2015} (see (\ref{eqn:vehicle})) are in the form (\ref{eqn:affine}). Throughout the paper, we will refer to the vehicle with dynamics given by (\ref{eqn:affine}) as {\em ego}.

\begin{definition}
\label{def:forwardinv}(\textit{Forward invariance} \cite{Nguyen2016}) A set $C\subset\mathbb{R}^{n}$ is forward invariant for
system (\ref{eqn:affine}) if $\bm x(0) \in C$ implies $\bm x(t)\in C,$ $\forall t\geq0$.
\end{definition}

\begin{definition}
\label{def:relative} (\textit{Relative degree} \cite{Nguyen2016}) The relative degree of a
(sufficiently many times) differentiable function $b:\mathbb{R}^{n}%
\rightarrow\mathbb{R}$ with respect to system (\ref{eqn:affine}) is the number
of times it needs to be differentiated along its dynamics (Lie derivatives) until the control
$\bm u$ explicitly shows in the corresponding derivative.
\end{definition}

In this paper, since function $b$ is used to define a constraint $b(\bm
x)\geq0$, we will also refer to the relative degree of $b$ as the relative
degree of the constraint. 

\subsection{High Order Control Barrier Functions}
\begin{definition}
\label{def:classk} (\textit{Class $\mathcal{K}$ function} \cite{Khalil2002}) A
continuous function $\alpha:[0,a)\rightarrow[0,\infty), a > 0$ is said to
belong to class $\mathcal{K}$ if it is strictly increasing and $\alpha(0)=0$.
\end{definition}
Given $b:\mathbb{R}^{n}\rightarrow\mathbb{R}$ and a constraint $b(\bm x)\geq0$ with relative
degree $m$, we define $\psi_{0}(\bm
x):=b(\bm x)$ and a sequence of functions $\psi_{i}:\mathbb{R}%
^{n}\rightarrow\mathbb{R},i\in\{1,\dots,m\}$:
\vspace{-2pt}
\begin{equation}
\begin{aligned} \psi_i(\bm x) := \dot \psi_{i-1}(\bm x) + \alpha_i(\psi_{i-1}(\bm x)),i\in\{1,\dots,m\}, \end{aligned} \label{eqn:functions}%
\end{equation}
where $\alpha_{i}(\cdot),i\in\{1,\dots,m\}$ denotes a $(m-i)^{th}$ order
differentiable class $\mathcal{K}$ function. We further define a sequence of sets $C_{i}, i\in\{1,\dots,m\}$ associated
with (\ref{eqn:functions}) in the following form:
\begin{equation}
\label{eqn:sets}\begin{aligned} C_i := \{\bm x \in \mathbb{R}^n: \psi_{i-1}(\bm x) \geq 0\}, i\in\{1,\dots,m\}. \end{aligned}
\end{equation}

\begin{definition}
\label{def:hocbf} (\textit{High Order Control Barrier Function (HOCBF)}
\cite{Xiao2019}) Let $C_{1}, \dots, C_{m}$ be defined by (\ref{eqn:sets}%
) and $\psi_{1}(\bm x), \dots, \psi_{m}(\bm x)$ be defined by
(\ref{eqn:functions}). A function $b: \mathbb{R}^{n}\rightarrow\mathbb{R}$ is
a High Order Control Barrier Function (HOCBF) of relative degree $m$ for
system (\ref{eqn:affine}) if there exist $(m-i)^{th}$ order differentiable
class $\mathcal{K}$ functions $\alpha_{i},i\in\{1,\dots,m-1\}$ and a class
$\mathcal{K}$ function $\alpha_{m}$ such that 
\begin{equation}
\label{eqn:constraint}\begin{aligned} \sup_{\bm u\in U}[L_f^{m}b(\bm x) + L_gL_f^{m-1}b(\bm x)\bm u + S(b(\bm x)) \\+ \alpha_m(\psi_{m-1}(\bm x))] \geq 0, \end{aligned}
\end{equation}
for all $\bm x\in C_{1} \cap,\dots, \cap C_{m}$. 
$L_{f}^{m}$ ($L_{g}$) denotes Lie derivatives along
$f$ ($g$) $m$ (one) times, and $S(\cdot)$ denotes the remaining Lie derivatives
along $f$ with degree less than or equal to $m-1$ (see \cite{Xiao2019} for more details).
\end{definition}

The HOCBF is a general form of the relative degree $1$ CBF \cite{Ames2017},
\cite{Glotfelter2017}, \cite{Lindemann2018} (setting $m=1$ reduces the HOCBF to
the common CBF form in \cite{Ames2017}, \cite{Glotfelter2017}, \cite{Lindemann2018}), and is also a general form of the exponential CBF
\cite{Nguyen2016}.

\begin{theorem}
\label{thm:hocbf} (\cite{Xiao2019}) Given a HOCBF $b(\bm x)$ from Def.
\ref{def:hocbf} with the associated sets $C_{1}, \dots, C_{m}$ defined
by (\ref{eqn:sets}), if $\bm x(0) \in C_{1} \cap,\dots,\cap C_{m}$,
then any Lipschitz continuous controller $\bm u(t)$ that satisfies
(\ref{eqn:constraint}) $\forall t\geq0$ renders $C_{1}\cap,\dots,
\cap C_{m}$ forward invariant for system (\ref{eqn:affine}).
\end{theorem}

\begin{definition}\label{def:clf} \textit{(Control Lyapunov Function (CLF) \cite{Aaron2012})} 
A continuously differentiable function $V :\mathbb{R}^{n}\rightarrow\mathbb{R}_{\ge0}$ is an exponentially stabilizing control Lyapunov function (CLF) if there exist positive constants $c_1 >0, c_2 > 0, c_3 > 0$ such that $\forall
\bm x\in X$, $c_{1}||\bm x||^{2} \leq V(\bm x)
\leq c_{2} ||\bm x||^{2}$, the following holds:
\begin{equation}\label{eqn:CLF}
\inf_{\bm u\in U} \lbrack L_{f}V(\bm x)+L_{g}V(\bm x)
\bm u + c_{3}V(\bm x)\rbrack\leq0.
\end{equation}
\end{definition} 

\begin{theorem} [\cite{Aaron2012}] \label{thm:clf}
	 Given a CLF as in Def. \ref{def:clf}, any Lipschitz continuous controller $ \bm u(t),\forall t\geq 0$ that satisfies (\ref{eqn:CLF}) 
	exponentially stabilizes system (\ref{eqn:affine}) to the origin.
\end{theorem}

\subsection{Quadratic Programming (QP) Approach}
\label{sec:qp-app}

Recent works \cite{Ames2017}, \cite{Lindemann2018}, \cite{Nguyen2016} combine CBFs and CLFs with quadratic costs to formulate an optimization problem that stabilizes a system using CLFs subject to safety constraints given by CBFs. Time is discretized and CBF and CLF constraints are considered at each discrete time step.  These constraints are linear in control since the state value is fixed at the beginning of the discretization interval. Therefore, in every interval, the optimization problem is a QP. The optimal control obtained by solving each QP is applied at the current time step and held constant for the whole interval. The next state is found by integrating  (\ref{eqn:affine}). The inter-sampling effect and the case of systems with unknown dynamics are solved in \cite{Xiao2021CDC}. The usefulness of this approach is conditioned upon the feasibility of the QP at every time step. In the case of constraints with high relative degrees, the CBFs can be replaced by HOCBFs \cite{Xiao2019}.

\section{PROBLEM FORMULATION AND APPROACH}
\label{sec:prob}

For a vehicle with dynamics (\ref{eqn:affine}) and starting at ${\bm x}(0)=\bm x_0$, consider an optimal control problem in the form:
\vspace{-3pt}
\begin{equation}\label{eqn:gcost}
\min_{\bm u(t)} \int_{0}^{T}J(||\bm u(t)||)dt,
\end{equation}
where $||\cdot||$ denotes the 2-norm of a vector, $T > 0$ denotes a bounded final time, and $J$ is a strictly increasing function of its argument (e.g., an energy consumption function $J(||\bm u(t)||) = ||\bm u(t)||^2$). We consider the following additional requirements:

\textbf{Trajectory tracking}: We require the vehicle to stay as close as possible to a desired {\em reference trajectory} $\mathcal{X}_r$ (e.g., middle of its lane).

\textbf{State constraints}: We impose a set of constraints (componentwise) on the state of system (\ref{eqn:affine}) in the following form:
\begin{equation}\label{eqn:state}
\bm x_{min} \leq \bm x(t)\leq \bm x_{max}, \forall t\in[0,T],
\end{equation}
where $\bm x_{max}: = (x_{max,1},x_{max,2},\dots,x_{max,n})\in \mathbb{R}^n$ and $\bm x_{min}: = (x_{min,1},x_{min,2},\dots,x_{min,n})\in \mathbb{R}^n$ denote the maximum and minimum state vectors, respectively. Examples of such constraints for a vehicle include maximum acceleration, maximum braking, and maximum steering rate. 
\textbf{Priority structure:}  We require the system trajectory $\mathcal{X}$ of (\ref{eqn:affine}) starting at $\bm x(0)=\bm x_0$ to satisfy a priority structure $\langle R,\sim_p,\leq_p\rangle$, i.e.:
\begin{equation}\label{eqn:rulebook-sat}
\mathcal{X}\models \langle R,\sim_p,\leq_p\rangle, 
\end{equation}
where $\sim_p$ is an equivalence relation over a finite set of rules $R$ and $\leq_p$ is a total order over the equivalence classes. 
For each rule from $R$, there exists a violation metric, which for a given trajectory gives a violation score that quantifies how much the trajectory violates the rule.
Our priority structure, called \underline{T}otal \underline{OR}der over e\underline{Q}uivalence classes (TORQ), is related to the rulebooks from \cite{Censi2020}. However, rather than allowing for a partial order over the set of rules $R$, we require that any two rules are either comparable or equivalent. 
A formal definition for a priority structure and its satisfaction will be given in Sec. \ref{sec:prio-struc}. 

\textbf{Control bounds}: We impose control bounds as given in (\ref{eqn:control}). Examples include jerk and steering acceleration.  

Formally, we can define the optimal control problem as follows:

\vspace{1mm}
\begin{problem}\label{prob:main}
	Find a control policy for system (\ref{eqn:affine}) such that the objective function in (\ref{eqn:gcost}) is minimized, and the trajectory tracking, state constraints (\ref{eqn:state}), the TORQ priority structure $\langle R,\sim_p,\leq_p\rangle$, and control bounds (\ref{eqn:control}) are satisfied by the generated trajectory starting at a given $\bm x(0)$. 
\end{problem}

We consider {\em offline} and {\em  online} versions of  Problem \ref{prob:main}. For the offline version, we assume that all the states of the other traffic participants (e.g., pedestrians, parked vehicles, other moving vehicles) are known within the defined time interval $[0, T]$. In the online version, we assume that ego only has information on the existence and state of traffic participants in its local sensing range. 

\textbf{Offline approach:} Our approach to the { offline} version of Problem \ref{prob:main} follows the general lines of the QP method described in Sec. \ref{sec:qp-app}, where the (quadratic) cost is given by (\ref{eqn:gcost}). We use CLFs to track each point on the reference trajectory $\mathcal{X}_r$ during each time interval, and 
HOCBFs to implement the state constraints (\ref{eqn:state}). Since the optimization is performed over a very short time interval (i.e., it is myopic), tracking the reference trajectory is difficult, and careful tuning of the parameters in the CLFs (such as $c_3$ in (\ref{eqn:CLF})) is necessary.
Possible approaches include machine learning 
\cite{Xiao2020CDC} and gradient descent. In this paper, we use 
gradient descent to find parameters that minimize the tracking error and maintain feasibility (myopic and aggressive tracking can lead to infeasibility of the QPs). Note that, while these can be expensive, all the computation is performed offline. We show that the satisfaction of the rules from $R$ can be written as forward invariance for sets described by differential functions, and enforce them using HOCBFs.  The control bounds (\ref{eqn:control}) are considered as constraints. We provide an iterative solution that determines the rule relaxation in advance to Problem \ref{prob:main}, where each iteration involves solving a sequence of QPs. In the first iteration, all the rules from $R$ are considered. If the corresponding QPs are feasible, then an optimal control is found. Otherwise, we solve the problem again
starting from the initial time by iteratively relaxing the satisfaction of rules from subsets of $R$ based on their priorities, and minimize the corresponding relaxations by including them in the cost function.

\textbf{Online approach:}  For online control, we formulate a receding horizon (Model Predictive Control, MPC) optimization problem, in which the reference trajectory tracking error is included in the cost. Although tracking is still sub-optimal, it is much better in terms of tracking accuracy than the CLF approach for offline control described above. Moreover, extensive parameter optimization, which is performed in the offline case to avoid aggressive control induced by the short optimization period, is not necessary in the online case, as the optimization is over a longer period. Finally, only the active rules (corresponding to detected instances) at a given time add constraints to the optimization problem in the online case. 
As it will become clear later, even though the online approach is more expensive at each time step, it avoids potentially many optimization iterations that are performed in the offline case.

\section{Rules and priority structures}
\label{sec:trb}

In this section, we extend the rulebooks from \cite{Censi2020} by formalizing the rules and defining violation metrics. We introduce the TORQ priority structure, in which all rules are comparable, and it is particularly suited for the hierarchical control framework proposed in Sec. \ref{sec:optim-rule-approx}. 
\subsection{Rules} 
In the definition below, an {\em instance} $i\in S_p$ is a traffic participant or artifact with the respect to whom a rule applies, where $S_p$ is the set of all instances. 
For example, in a rule to maintain clearance from pedestrians, a pedestrian is an instance, and there can be many instances encountered by ego in a given scenario. Instances can also be traffic artifacts like lane boundaries or stop lines. 

\begin{definition} (Rule)\label{def:rule}
A rule is composed of a statement and three violation metrics.
A rule statement is a Boolean formula that ego should satisfy to the greatest extent possible for all times.~\footnote{This assumption applies to all the rules that we consider in this paper. However, there might exist rules (e.g., ``reach goal") for which satisfaction is required at some times.} 
A formula is inductively defined as:
\begin{equation} \label{eqn:task}
\varphi := \mu\vert \neg \varphi \vert \varphi_1\wedge \varphi_2,
\end{equation}
where $\varphi,\varphi_1,\varphi_2$ are formulas, $\mu :=(h(\bm x)\geq 0)$ is a predicate on the state vector $\bm x$ of system (\ref{eqn:affine}) with $h:\mathbb{R}^n\rightarrow \mathbb{R}$. $\wedge, \neg$ are Boolean operators for conjunction and negation, respectively. The three violation metrics for a rule $r$ are defined as:
 \begin{enumerate}
     \item instantaneous violation metric $\varrho_{r,i}(\bm x(t)) \in [0,1],$
     \item instance violation metric $\rho_{r,i}(\mathcal{X})\in [0,1]$, and
     \item total violation metric $P_{r}(\mathcal{X})\in [0,1]$,
 \end{enumerate}
  where $i$ is an instance, $\bm{x}(t)$ is a trajectory at time $t$ and $\mathcal{X}$ is a whole trajectory of ego. 
 The instantaneous violation metric $\varrho_{r,i}(\bm x(t))$
 quantifies violation by a trajectory at a specific time $t$ with respect to a given instant $i$. The instance violation metric $\rho_{r,i}(\mathcal{X})$ captures violation with respect to a given instance $i$ over the whole time of a trajectory, and is obtained by aggregating $\varrho_{r,i}(\bm x(t))$ over the entire time of a trajectory $\mathcal{X}$. 
 The total violation metric $P_{r}$ is the aggregation of the instance violation metric $\rho_{r,i}(\mathcal{X})$ over all instances $i\in S_p$. 
 \end{definition}
 
 The aggregations in the above definitions can be implemented through selection of a maximum or a minimum, integration over time, summation over instances, or by using general $L_p$ norms. A zero value for a violation score shows satisfaction of the rule. A strictly positive value denotes violation - the larger the score, the more ego violates the rule. 
 Throughout the paper, for simplicity, we use $\varrho_{r}$ and $\rho_{r}$ instead of $\varrho_{r,i}$ and $\rho_{r,i}$ if there is only one instance. Examples of rules (statements and violations metrics and scores) are given in Sec. \ref{sec:case} and in the Appendix. 

We divide the set of rules into two categories: (1) {\em clearance rules} - rules enforcing that ego maintains a minimal distance to other traffic participants and to the side of the road or lane (2) {\em non-clearance rules} - rules that are not contained in the first category, such as speed limit rules. Clearance rules relate to safety in the sense that satisfaction of these rules for all times guarantees absence of collisions (although violations of these rules do not necessarily mean that a collision will occur). In Sec. \ref{sec:rule-approx}, we provide a general methodology to express clearance rules as inequalities involving differentiable functions, which will allow us to enforce their satisfaction using HOCBFs.

 \begin{remark}
 The violation metrics from Def. \ref{def:rule} are inspired from Signal Temporal Logic (STL) robustness \cite{Maler2004,donze,mehdipour2019agm}, which quantifies how a signal (trajectory) satisfies a temporal logic formula. In this paper, we focus on rules that we aim to satisfy for all times. Therefore, the rules in (\ref{eqn:task}) can be seen as (particular) STL formulas, which all start with an ``always" temporal operator (omitted here).  
 \end{remark}

\subsection{Priority Structure}
\label{sec:prio-struc}

The pre-order rulebook from \cite{Censi2020}
defines a ``base" pre-order that captures relative priorities of some (comparable) rules, which are often similar in different states and countries. 
 A pre-order rulebook can be made more precise for a specific jurisdiction by adding rules and/or priority relations through priority refinement, rule aggregation and augmentation. This can be done through empirical studies or learning from local data to construct a total order rulebook. To order trajectories, since in a pre-order rules can be incomparable, the authors of \cite{Censi2020} enumerated all the total orders compatible with a given pre-order. In this paper, motivated by the hierarchical control framework
described in Sec. \ref{sec:optim-rule-approx}, we require that any two rules are in a relationship, in the sense that they are either equivalent or comparable with respect to their priorities. 

\begin{definition} [TORQ Priority Structure]
A \underline{T}otal \underline{OR}der over e\underline{Q}uivalence classes (TORQ) priority structure is a tuple $\langle R,\sim_p,\leq_p\rangle$, where $R$ is a finite set of rules, $\sim_p$ is an equivalence relation over $R$, and $\leq_p$ is a total order over the set of equivalence classes determined by $\sim_p$. 
\end{definition}

Equivalent rules (i.e., rules in the same class) have the same priority. Given two equivalence classes $\mathcal{O}_1$ and $\mathcal{O}_2$ with $\mathcal{O}_1\leq_p \mathcal{O}_2$, every rule $r_1\in \mathcal{O}_1$ has lower priority than every rule $r_2\in \mathcal{O}_2$. Since $\leq_p$ is a total order, any two rules $r_1,r_2\in R$ are comparable, in the sense that exactly one of the following three statements is true: (1) $r_1$ and $r_2$ have the same priority, (2) $r_1$ has (strictly) higher priority than $r_2$, and (3) $r_2$ has (strictly) higher priority than $r_1$. 

Given a TORQ $\langle R,\sim_p,\leq_p\rangle$, we can assign numerical (integer) priorities to the rules. We assign priority 1 to the equivalence class with the lowest priority, priority 2 to the next one and so on. 
{\em The rules inside an equivalence class inherit the priority from their equivalence class}. Given a priority structure $\langle R,\sim_p,\leq_p\rangle$ and violation scores for the rules in $R$, we can compare trajectories:

\begin{definition}[Trajectory Comparison] \label{def:traj_cmp}
A trajectory $\mathcal{X}_1$ is said to be {\bf better} (less violating) than another trajectory $\mathcal{X}_2$
if the highest priority rule(s) violated by $\mathcal{X}_1$ has a lower priority than the highest priority rule(s) violated by $\mathcal{X}_2$. If both trajectories violate equivalent highest priority rules, then the one with the smaller maximum of the total violation scores of the highest priority violated rules is better. If the violation scores are the same, we consider the violation scores of the rules with lower priority, and reiterate.
\end{definition}
\vspace{-2pt}
It is easy to see that, by following Def. \ref{def:traj_cmp}, given two trajectories, one is always better than the other. In theory, two trajectories could also be equivalent if $\mathcal{X}_1$ is not better than $\mathcal{X}_2$ and $\mathcal{X}_2$ is not better than $\mathcal{X}_1$ (i.e., the two trajectories only violate equivalent rules and with the same scores). However, since scores are continuous, in practice it is very unlikely that two scores be equal, and we do not consider equivalent trajectories. 

\begin{example}\label{ex:three-traj}
Consider the driving scenario from Fig. \ref{fig:autonomous1} and TORQ $\langle R,\sim_p,\leq_p\rangle$ in Fig. \ref{fig:rulebook1}, where $R = \{r_7, r_3, r_5, r_6\}$,  and $r_7$: ``Clearance with parked vehicles'', $r_3$: ``Lane keeping'', $r_5$: ``Speed limit'' and $r_6$: ``Comfort''. There are 3 equivalence classes given by $\mathcal{O}_1=\{r_7\}$, $\mathcal{O}_2=\{r_3,r_5\}$ and $\mathcal{O}_3=\{r_6\}$. Rule $r_7$ has priority 1, $r_3$ and $r_5$ have priority 2, and $r_6$ has priority 3. Assume the instance (same as total, as there is only one instance for each rule) violation scores of rules $r_i$,  $i=7,3,5,6$ by trajectories $a,b,c$ are given by $\rho_i=(\rho_i(a),\rho_i(b),\rho_i(c))$ as shown in Fig. \ref{fig:rulebook1}. 
Based on Def. \ref{def:traj_cmp}, 
trajectory $b$ is better (less violating) than trajectory $a$ since the highest priority rule violated by $b$ ($r_3$) has a lower priority than the highest priority rule violated by $a$ ($r_7$). The same argument holds for trajectories $a$ and $c$, i.e., $c$ is better than $a$. The highest priority rules violated by trajectories $b$ and $c$ have the same priorities. Since the maximum violation score of the highest priority rules violated by $b$ is smaller than that for $c$, i.e., $\max(\rho_3(b),\rho_5(b))=0.1$, $\max(\rho_3(c),\rho_5(c))=0.4$, trajectory $b$ is better than $c$.
\end{example}
\begin{definition} (TORQ satisfaction) \label{def:rb_satisfy}
A trajectory $\mathcal{X}$ of system (\ref{eqn:affine}) starting at $\bm x(0)$
satisfies a TORQ
$\langle R,\sim_p,\leq_p\rangle$ (i.e., $\mathcal{X}\models \langle R,\sim_p,\leq_p\rangle$), if there are no better trajectories of (\ref{eqn:affine}) starting at $\bm x(0)$.
\end{definition}

Def. \ref{def:rb_satisfy} is central to our solution to Problem \ref{prob:main} (see Sec. \ref{sec:optim-rule-approx}), which is based on an iterative relaxation of the rules according to their satisfaction of the TORQ. 

\begin{figure}[htb]
	\centering
	\subfigure[Possible trajectories]{
		\begin{minipage}[t]{0.45\linewidth}
			\centering
			\includegraphics[scale=0.45]{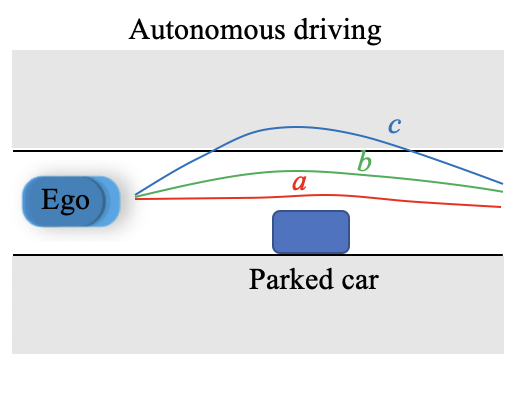} 
			\vspace{-1.8mm}
			\label{fig:autonomous1}%
		\end{minipage}%
	}	
	\subfigure[TORQ with violation scores.]{
		\begin{minipage}[t]{0.45\linewidth}
			\centering
			\includegraphics[scale=0.25]{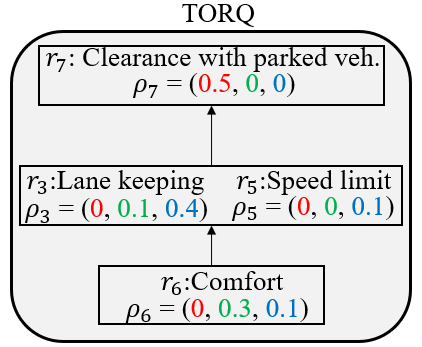} 
			\label{fig:rulebook1}%
		\end{minipage}%
	}	
	\centering
	\vspace{-6pt}
	\caption{An autonomous driving scenario with three possible trajectories, 4 rules, and 3 equivalence classes. The colors for the scores correspond to the colors of the trajectories. The rectangles show the equivalence classes}
\end{figure}

\section{CLEARANCE AND OPTIMAL DISK COVERAGE} 
\label{sec:rule-approx}

Satisfaction of a priority structure can be enforced by formulating real-time constraints on ego state $\bm x(t)$ that appear in the violation metrics. Satisfaction of the non-clearance rules can be easily implemented using HOCBFs (See Sec. \ref{sec:optim-rule-approx}, Sec. \ref{sec:app-rule-def}). For clearance rules, we define a notion of clearance region around ego and around the traffic participants in $S_p$ with respect to whom the rule applies (e.g., pedestrians and other vehicles). 

The clearance region for ego is defined as a rectangle with tunable speed-dependent lengths (i.e., we may choose to have a larger clearance from pedestrians when ego is driving with higher speeds) and defined based on ego footprint and functions $h_f(\bm x), h_b(\bm x), h_l(\bm x), h_r(\bm x)$ that determine the front, back, left, and right clearances as illustrated in Fig. \ref{fig:approx}, where $h_f,h_b,h_l,h_r:\mathbb{R}^n\rightarrow \mathbb{R}_{\geq0}$. The clearance regions for participants (instances) are defined such that they comply with their geometry and cover their footprints, e.g., (fixed-length) rectangles for other vehicles and (fixed-radius) disks for pedestrians, as shown in Fig. \ref{fig:approx}.

To satisfy a clearance rule involving traffic participants, we need to avoid any overlaps between the clearance regions of ego and traffic participants. 

We define a function $d_{min}(\bm x, \bm x_i): \mathbb{R}^{n+n_i} \rightarrow \mathbb{R}$ to determine the signed distance between the clearance regions of ego and participant $i\in S_p$ ($\bm x_i\in\mathbb{R}^{n_i}$ denotes the state of participant $i$), which is negative if the clearance regions overlap. Therefore, satisfaction of a clearance rule can be imposed by having a constraint on $d_{min}(\bm x, \bm x_i)$ to be non-negative. For the clearance rules ``stay in lane" and ``stay in drivable area", we require that ego clearance region be within the lane and the drivable area, respectively.

However, finding $d_{min}(\bm x, \bm x_i)$ can be computationally expensive. For example, the distance between two rectangles could be from corner to corner, corner to edge, or edge to edge. Since each rectangle has $4$ corners and $4$ edges, there are 64 possible cases. More importantly, this computation leads to a non-smooth $d_{min}(\bm x, \bm x_i)$ function, which cannot be used to enforce clearance using a CBF approach. To address these issues, we propose an optimal coverage of the rectangles with disks, which allows mapping the clearance rules to a set of smooth HOCBF constraints (i.e., there will be one constraint for each pair of centers of disks pertaining to different traffic participants). 

We use $l > 0$ and $w > 0$ to denote the length and width of ego's footprint, respectively. Assume we use $z \in\mathbb{N}$ disks with centers located on the center line of the clearance region to cover it (see Fig. \ref{fig:proof}). Since all the disks have the same radius, the minimum radius to fully cover ego's clearance region, denoted by $ \mathfrak{r}>0$, is given by:
\begin{equation}\label{eqn:radius}
    \mathfrak{r} = \sqrt{\left(\frac{w + h_l(\bm x) + h_r(\bm x)}{2} \right)^2 + \left(\frac{l+h_f(\bm x) + h_b(\bm x)}{2z}\right)^2}.
\end{equation}
The minimum radius $\mathfrak{r}_i$ of the rectangular clearance region for a traffic participant $i \in S_p$ with disks number $z_i$ is defined in a similar way using the length and width of its footprint and setting $h_l,h_r,h_b,h_f=0$.

\begin{figure}[thpb]
	\centering
	\includegraphics[scale=0.30]{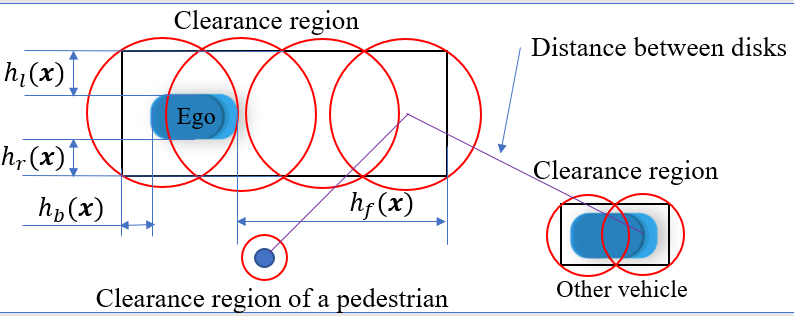}
	\caption{The clearance regions and their coverage with disks: the clearance region and the disks are speed dependent for ego and fixed for the other vehicle and the pedestrian. We consider the distances between all the possible pairs of disks from ego and other traffic participants (e.g., pedestrians, parked vehicles). There are 12 distance pairs in total, and we only show two of them w.r.t. the pedestrian and another vehicle, respectively.}
	\label{fig:approx}%
\end{figure}
\begin{figure}[!bht]
	\centering
	\includegraphics[scale=0.30]{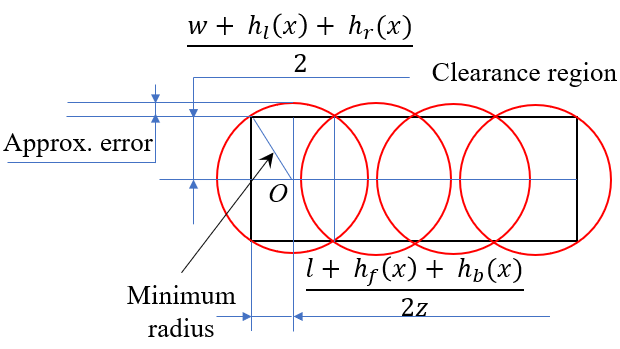}
	\vspace{-3pt}
	\caption{The optimal disk coverage of a clearance region.}
	\label{fig:proof}%
\end{figure}

Assume the center of the disk $j\in \{1,\dots,z\}$ for ego, and the center of the disk $k\in \{1,\dots,z_i\}$ for the instance $i \in S_p$ are given by $(x_{e,j}, y_{e,j}) \in \mathbb{R}^2$ and $(x_{i,k}, y_{i,k})\in \mathbb{R}^2$, respectively (See Appendix \ref{sec:app-coverage}). To avoid any overlap between the corresponding disks of ego and the instance $i\in S_p$, we impose the following constraints:
\begin{equation}\label{eqn:rule_cons}
\begin{aligned}
        \sqrt{(x_{e,j} - x_{i,k})^2 + (y_{e,j} - y_{i,k})^2} \geq   \mathfrak{r} + \mathfrak{r}_i ,\\ \forall j\in\{1,\dots,z\}, \forall k\in\{1,\dots,z_i\}. 
    \end{aligned}
\end{equation}
Since disks fully cover the clearance regions, enforcing \eqref{eqn:rule_cons} also guarantees that $d_{min}(\bm x, \bm x_i)\ge0$. For the clearance rules ``stay in lane" and ``stay in drivable area", we can get similar constraints as (\ref{eqn:rule_cons}) to make the disks that cover ego's clearance region stay within them (e.g., we can consider $h_l,h_r,h_b,h_f=0$ and formulate \eqref{eqn:rule_cons} such that the distance between ego disk centers and the line in the middle of ego's current lane be less than $\frac{w_l}{2} - \mathfrak{r}$, where $w_l > 0$ denotes the lane width). 
Thus, we can formulate satisfaction of all the clearance rules as continuously differentiable constraints (\ref{eqn:rule_cons}), and implement them using HOCBFs.

To efficiently formulate the proposed optimal disk coverage approach, we need to find the minimum number of disks that fully cover the clearance regions as it determines the number of constraints in \eqref{eqn:rule_cons}. Moreover, we need to minimize the lateral approximation error since large errors imply overly conservative constraints (See Fig. \ref{fig:proof}). This can be formally defined as an optimization problem, and solved offline 
to determine the numbers and radii of the disks in \eqref{eqn:rule_cons} (the details are provided in Appendix \ref{sec:app-coverage}).

\section{OFFLINE SOLUTION TO PROBLEM \ref{prob:main}}
\label{sec:oc}

\subsection{Trajectory Tracking}\label{sec:tracking}

As discussed in Sec. \ref{sec:vd}, (\ref{eqn:affine}) can define ``traditional" vehicle dynamics with respect to an inertial reference frame
\cite{Ames2017}, or dynamics defined along a given reference trajectory \cite{Rucco2015} (see (\ref{eqn:vehicle})). The case study considered in this paper falls in the second category (the middle of ego's initial 
lane is the default reference trajectory). We use the model from \cite{Rucco2015}, in which part of the state of (\ref{eqn:affine}) captures the tracking errors with respect to the reference trajectory. 
The tracking problem then becomes stabilizing the error states to 0. Suppose the error state vector is $\bm y\in{R}^{n_0}, n_0 \leq n$ (the components in $\bm y$ are part of the components in $\bm x$). We define a CLF $V(\bm x) = ||\bm y||^2$ ($c_3 = \epsilon > 0$ in Def. \ref{def:clf}). Any control $\bm u$ that satisfies the relaxed CLF constraint \cite{Ames2017} given by:
\begin{equation} \label{eqn:clf1}
     L_{f}V(\bm x)+L_{g}V(\bm x)
\bm u + \epsilon V(\bm x)\leq \delta_e,
\end{equation} 
exponentially stabilizes the error states to 0 if $\delta_e(t) = 0, \forall t\in[0,T]$, where $\delta_e>0$ is a relaxation variable that will be used in the cost to compromise between stabilization and feasibility. Note that the CLF constraint (\ref{eqn:clf1}) only works for $V(\bm x)$ with relative degree 1. If the relative degree is larger than $1$, we can use input-to-state linearization and state feedback control \cite{Khalil2002} to reduce the relative degree to one \cite{Xiao2020}.

\subsection{Offline Optimal Control}
\label{sec:optim-rule-approx}

In this section, we present our complete framework to solve the offline version of
Problem \ref{prob:main}. In this case, the information about all the instances is available for all times, and thus, all the rules are considered for all times between $0$ and $T$.
We propose a recursive algorithm to iteratively relax the satisfaction of the rules in the priority structure $\langle R,\sim_p,\leq_p\rangle$ (if needed) based on the total order over the equivalence classes. 

Let $R_\mathcal{O}$ be the set of equivalence classes in $\langle R,\sim_p,\leq_p\rangle$, and $N_\mathcal{O}$ be the cardinality of $R_\mathcal{O}$. 
We construct the power set of equivalence classes denoted by $S = 2^{R_\mathcal{O}}$, and incrementally (from low to high priority) sort the sets in $S$ based on the highest priority of the equivalence classes in each set according to the total order and denote the sorted set by $S_{sorted} = \{S_1, S_2, \dots, S_{2^{N_\mathcal{O}}}\}$, where $S_1 =\{ \emptyset\}$. We use this sorted set in our optimal control formulation to obtain satisfaction of the higher priority classes, even at the cost of relaxing satisfaction of the lower priority classes. Therefore, from Def. \ref{def:rb_satisfy}, the solution of the optimal control will satisfy the priority structure. 
\begin{example}\label{exm:sorted}
Reconsider Exm. \ref{ex:three-traj}. 
We define $R_\mathcal{O} = \{ \mathcal{O}_1,\mathcal{O}_2,\mathcal{O}_3\}$. 
Based on the given total order $\mathcal{O}_1\leq_p \mathcal{O}_2 \leq_p \mathcal{O}_3$, we can write the sorted power set as $S_{sorted} = \left \{\right.\!\{\emptyset\}, \{\mathcal{O}_1\},\{\mathcal{O}_2\},\{\mathcal{O}_1,\mathcal{O}_2\},\{\mathcal{O}_3\}, \{\mathcal{O}_1,\mathcal{O}_3\},\{\mathcal{O}_2,\mathcal{O}_3\},\\ \{\mathcal{O}_1,\mathcal{O}_2,\mathcal{O}_3\} \}$. 
\end{example}

In order to find a trajectory that satisfies a given TORQ, we first assume that all the rules are satisfied. Starting from $S_1=\{\emptyset\}$ in the sorted set $S_{sorted}$, we solve Problem \ref{prob:main} given that no rules are relaxed, i.e., all the rules must be satisfied. If the problem is infeasible, we move to the next set $S_2 \in S_{sorted}$, and relax all the rules of all the equivalence classes in $S_2$ while enforcing satisfaction of all the other rules in the equivalence class set denoted by $R_\mathcal{O} \setminus S_2$. This procedure is done recursively until we find a feasible solution of Problem \ref{prob:main}.

Formally, at $k = 1,2\dots, 2^{N_\mathcal{O}}$ for $S_k\in S_{sorted}$, we relax all the rules $i\in \mathcal{O}$ for all the equivalence classes $\mathcal{O} \in S_k$ and reformulate Problem \ref{prob:main} as the following optimal control problem:
\begin{equation}
\min_{\bm u,\delta_e, {\delta_i}_{, i\in \mathcal{O}, \mathcal{O}\in S_k}} \int_{0}^{T}J(||\bm u||) +  p_e\delta_e^2 +\sum_{i\in \mathcal{O}, \mathcal{O}\in S_k}p_i \delta_i^2dt \label{eqn:cost2}
\end{equation}
subject to:\\
\text{\qquad dynamics (\ref{eqn:affine}), control bounds (\ref{eqn:control}), CLF constraint (\ref{eqn:clf1}),}
\begin{align}
&\begin{aligned}L_{f}^{m_{j}}b_{j}(\bm x)+L_{g}L_{f}^{m_{j}-1}b_{j}(\bm x)\bm
u+S(b_{j}(\bm x))&\\+\alpha_{m_j}(\psi_{m_{j}-1}(\bm x))&\geq0, \\
\forall j\in \mathcal{O}, \forall \mathcal{O}\in R_{\mathcal{O}}\setminus S_k,\end{aligned}\label{eqn:optim-not-relax-rules}
\\
&\begin{aligned}L_{f}^{m_{i}}b_{i}(\bm x)+L_{g}L_{f}^{m_{i}-1}b_{i}(\bm x)\bm
u+S(b_{i}(\bm x))&\\+\alpha_{m_i}(\psi_{m_{i}-1}(\bm x))&\geq \delta_i,\\
\forall i\in \mathcal{O}, \forall \mathcal{O}\in S_k,\end{aligned}\label{eqn:optim-relax-rules}
\\
&\begin{aligned}
L_{f}^{m_{l}}b_{l}(\bm x)+L_{g}L_{f}^{m_{l}-1}b_{lim,l}(\bm x)\bm
u&+S(b_{lim,l}(\bm x))\\+\alpha_{m_l}(\psi_{m_{l}-1}(\bm x))&\geq0,\forall l\in \{1,\dots, 2n\},
\end{aligned} \label{eqn:optim-state-cons}
\end{align}
where $p_e > 0$ and $p_i>0,  i\in \mathcal{O}, \mathcal{O}\in S_k$ assign the trade-off between the CLF relaxation $\delta_e$ (used for trajectory tracking) and the HOCBF relaxations $\delta_i$.
$m_i,m_j,m_l$ denote the relative degrees of $b_i(\bm x),b_j(\bm x),b_{lim,l }(\bm x)$, respectively. The functions $b_i(\bm x)$ and $b_j(\bm x)$ are HOCBFs for the rules in $\langle R,\sim_p,\leq_p\rangle$, and are implemented directly from the rule statement for non-clearance rules or by using the optimal disk coverage framework for clearance rules. At relaxation step $k$, HOCBFs corresponding to the rules in $\mathcal{O}$, $\forall\mathcal{O}\in S_k$ are relaxed by adding $p_i>0,  i\in \mathcal{O}, \mathcal{O}\in S_k$ in \eqref{eqn:optim-relax-rules}, while for other rules in $R$ in \eqref{eqn:optim-not-relax-rules} and the state constraints \eqref{eqn:optim-state-cons}, regular HOCBFs are used. We assign $p_i,  i\in \mathcal{O}, \mathcal{O}\in S_k$ according to their relative priorities, i.e., we choose a larger $p_i$ for the rule $i$ that belongs to a higher priority class. The above optimization problem is solved from time $0$ to $T$ for each iteration using the time discretization method described in Sec. \ref{sec:qp-app}. 

The functions $b_{lim,l}(\bm x), l\in\{1,\dots,2n\}$ are HOCBFs for the state limitations (\ref{eqn:state}). The functions $\psi_{m_i}(\bm x), \psi_{m_j}(\bm x), \psi_{m_l}(\bm x)$ are defined as in (\ref{eqn:functions}). $\alpha_{m_i},\alpha_{m_j},\alpha_{m_l}$ are 
penalized (i.e., multiplied by scalars)
to improve the feasibility of the problem above \cite{Xiao2019,Xiao2020CDC}.

If the above optimization problem is feasible for all $t\in[0,T]$, we can specifically determine which rules (within an equivalence class) are relaxed based on the values of $\delta_i, i\in \mathcal{O}, \mathcal{O}\in S_k$ in the optimal solution (i.e., if $\delta_i(t) = 0, \forall t\in\{0,T\}$, then rule $i$ does not need to be relaxed). This procedure is summarized in Alg. \ref{alg:sort}. 

\begin{remark}[Complexity]
The optimization problem (\ref{eqn:cost2}) is solved using QPs as described at the end of Sec.~\ref{sec:pre}. The complexity of a QP is $O(y^3)$, where $y\in\mathbb{N}$ is the number of decision variables. The total time for each iteration $k\in\{1,\dots, 2^{N_{\mathcal{O}}}\}$ depends on the discretization time $\Delta t$ and the final time $T$. The worst (highest) possible number of QPs  is $2^{N_{\mathcal{O}}}(T/\Delta t)\zeta$, where $\zeta\in\mathbb{N}$ is the number of parameter optimizations performed at each iteration to achieve a desired tracking accuracy while maintaining feasibility. Some running times are given in Appendix \ref{sec:tool}. 
\end{remark}

\subsection{Pass/Fail Evaluation}\label{sec:p/f}
As an extension to the offline version of Problem \ref{prob:main}, we formulate and solve a pass/fail (P/F) procedure, in which we are given a vehicle trajectory, and the goal is to accept (pass, P) or reject (fail, F) it based on the satisfaction of the rules. Specifically, given a candidate trajectory $\mathcal{X}_c$ of system (\ref{eqn:affine}), and given a TORQ $\langle R,\sim_p,\leq_p\rangle$, we pass (P) $\mathcal{X}_c$ if we cannot find a better trajectory according to Def. \ref{def:traj_cmp}. Otherwise, we fail (F) $\mathcal{X}_c$. 

We proceed as follows: we find the total violation scores of the rules in $\langle R,\sim_p,\leq_p\rangle$ for the candidate trajectory $\mathcal{X}_c$. If no rules in $R$ are violated, then no less-violating trajectory exists and we pass the candidate trajectory. Otherwise, 
we investigate the existence of a better (less violating) trajectory. We take the middle of ego's initial lane as the reference trajectory $\mathcal{X}_r$
and re-formulate the optimal control problem in (\ref{eqn:cost2}) to recursively relax rules such that if the optimization is feasible, the generated trajectory is better than the candidate trajectory $\mathcal{X}_c$. Specifically, assume that 
the highest priority rule(s) that the candidate trajectory $\mathcal{X}_c$ 
violates belongs to $\mathcal{O}_H$, $H \in\mathbb{N}$. Let $R_H\subseteq R_{\mathcal{O}}$ denote the set of equivalence classes with priorities not larger than $H$, and $N_H \in\mathbb{N}$ denote the cardinality of $R_H$. We construct a power set $S_{H} = 2^{R_H}$, and then apply Alg. \ref{alg:sort} , in which we replace $R_{\mathcal{O}}$ by $R_H$. We start to relax the highest priority equivalent classes in $S_H$ in order to quickly find a solution (as the rules in higher priority equivalence classes are usually related to clearance, and it is more likely to find a solution by relaxing them). \vspace{-3pt}
\begin{remark}\label{remark:cond-pass}
The procedure described above would fail a candidate trajectory $\mathcal{X}_c$ even if only a slightly better alternate trajectory (i.e., violating rules of the same highest priority but with slightly smaller violation scores) can be found by solving the optimal control problem. In practice, this might lead to an undesirably high failure rate. One way to deal with this, which we will consider in future work (see Sec. \ref{sec:con}), is to allow for more classification categories, e.g., ``Provisional Pass" (PP), which can then trigger further investigation of $\mathcal{X}_c$.   
\end{remark}

\begin{example}
Reconsider Exm. \ref{ex:three-traj} and assume trajectory $b$ is a candidate trajectory which violates rules $r_3, r_6$, 
thus, the highest priority rule that is violated by trajectory $b$ belongs to $\mathcal{O}_2$.  
We construct $R_H = \{ \mathcal{O}_1,\mathcal{O}_2\}$. 
The power set $S_H=2^{R_H}$ is then defined as $S_H = \{ \{\emptyset\},\{\mathcal{O}_1\}, \{\mathcal{O}_2\},\{\mathcal{O}_1,\mathcal{O}_2\}\}$, and is sorted based on the total order as $S_{H_{sorted}} = \{\{\emptyset\}, \{\mathcal{O}_1\},\{\mathcal{O}_2\}, \{\mathcal{O}_1,\mathcal{O}_2\}\}$.
\end{example}
\begin{algorithm}
\caption{Offline optimal control} \label{alg:sort}
	\KwIn{System (\ref{eqn:affine}) with $\bm x(0)$, cost function (\ref{eqn:gcost}), control bound (\ref{eqn:control}), state constraint (\ref{eqn:state}), TORQ $\langle R,\sim_p,\leq_p\rangle$, reference trajectory $\mathcal{X}_r$}
	\KwOut{Optimal ego trajectory and set of relaxed rules}
	1. Construct the power set of equivalence classes $S = 2^{R_\mathcal{O}}$\;
	2. Sort the sets in $S$ based on the highest priority of the equivalence classes in each set according to the total order and get $S_{sorted} = \{S_1, S_2, \dots, S_{2^{N_\mathcal{O}}}\}$\;
	3. $k = 0$\;
	\While{$k++\leq 2^{N_\mathcal{O}}$
	}{
    	Solve (\ref{eqn:cost2}) s.t. (\ref{eqn:affine}), (\ref{eqn:control}), (\ref{eqn:clf1}),  (\ref{eqn:optim-relax-rules}), (\ref{eqn:optim-not-relax-rules}) and (\ref{eqn:optim-state-cons}) from 0 to $T$\; 
    	\If{the above problem is feasible for all $t\in[0,T]$}{
    	    Generate the optimal trajectory $\mathcal{X}^*$ from \eqref{eqn:affine}\;
        	Construct relaxed set $R_{relax} = \{i: i\in \mathcal{O}, \mathcal{O}\in S_{k}\}$\;
        	\If{$\delta_i(t) = 0, \forall t\in[0,T]$}{
        	    Remove $i$ from $R_{relax}$\;}
        	break\;
    	}
	}
	4. Return $\mathcal{X}^*$ and $R_{relax}$\;
\end{algorithm}

\section{ONLINE SOLUTION TO PROBLEM \ref{prob:main}}
\label{sec:online-solution}

In this section, we consider the online version of Problem \ref{prob:main}, in which ego has access to local information only (e.g., the existence and states of instances) from its on-board sensors.

We propose a model predictive control (MPC) method that performs optimization over a finite horizon, and improves tracking performance, without the need to extensively tune the parameters in the optimization. We use the Adomian Decomposition Method (ADM) \cite{Zhang2012} to discretize dynamics (\ref{eqn:affine}) and obtain  the predictive model:
\begin{equation}\label{eqn:predictive}
    \bm x(k+1) = \hat f(\bm x(k)) + \hat g(\bm x(k))\bm u(k),
\end{equation}
where $\hat f:\mathbb{R}^{n}\rightarrow\mathbb{R}^{n}$, $\hat g:\mathbb{R}^{n}\rightarrow\mathbb{R}^{n\times q}$, $k\in\{0,1,\dots\}$ denotes the discrete time instants. The discretization time is denoted by $\Delta T$.~\footnote{To provide a fair comparison with the offline method, in Sec. \ref{sec:case}, we assume that $\Delta T$ is the same as the discretization interval used in the QP-approach to the offline version of the problem.} The accuracy of the predictive model above with respect to (\ref{eqn:affine}) depends on the approximation order we take in ADM. 

To keep computational complexity low, we propose a {\em two-step approach} for the  online version of Problem \ref{prob:main}. In the {\em first step}, we minimize the tracking error and the original cost $J$ (Sec. \ref{sec:online_track}). In the {\em second step}, we account for rule satisfaction by using CBFs and HOCBFs (Sec. \ref{sec:online-opt-ctrl}). 

\subsection{Trajectory Tracking}
\label{sec:online_track}

The CLF tracking approach introduced in the last section for offline control is very myopic as the optimization is performed 
at every time step over each (short) discretization time interval. As a result, it is sub-optimal, and ego tends to be aggressive to eliminate tracking errors, which could adversely lead to larger errors in the following time interval. This aggressiveness heavily depends on the parameters of the CLFs (e.g., $\epsilon$ and $p_e$, etc.), which can be properly tuned in offline control (see Sec. \ref{sec:optim-rule-approx}). However, finding such parameters in the short update of the online implementation is infeasible. We propose a receding horizon (MPC) approach that addresses this issue.  

To track the reference trajectory, we use a metric $C_1:\mathbb{R}^n\rightarrow\mathbb{R}^{\geq 0}$, which is defined at each state $\bm x$ (if $C_1(\bm x) = 0$, then the tracking error is 0 at $\bm x$; large $C_1(\bm x)$ denotes large tracking error). We solve the following optimization problem with receding horizon $H\in\mathbb{N}$ at each time $t\geq 0$ :
\begin{equation}\label{eqn:mpc}
    (\bm u_{mpc},\bm x_{mpc}) \!\!=\!\! \arg \!\!\!\!\! \min_{\bm u(0:H-1), \bm x(1:H)} \sum_{i = 1}^{H}C_1(\bm x(i)) \!+\!\! \sum_{i = 0}^{H-1}J(||\bm u(i)||)
\end{equation}
s.t. (\ref{eqn:state}) (evaluated at $1,\ldots,H$), (\ref{eqn:control}), and (\ref{eqn:predictive}) (for $k\in\{0,\dots, H-1\}$). The initial state $\bm x(0)$ in (\ref{eqn:mpc}) is set to $\bm x(t)$ at time instant $t\geq 0$. Note that the above optimization is, in general, a nonlinear program (NLP).

If the reference trajectory is given as a sequence of points, $C_1(\bm x)$ can be obtained through regression. If  dynamics (\ref{eqn:affine}) is defined with respect to the reference trajectory, we can find $C_1(\bm x)$ by regression of the reference trajectory curvature. 
An example can be found in the case study Section \ref{sec:case}, in which the road curvature is captured in the vehicle dynamics. We perform regression over the approximate curvature of discrete trajectory points, and the tracking errors, which are system states, can be derived from the regression within horizon $H$. 

Note that the constraints corresponding to rule satisfactions could be, in principle, incorporated in the above optimization problem (similar to the procedure for offline control).  
However, this optimization problem can easily become infeasible, especially when many constraints are active. Therefore, as already stated at the beginning of Sec. \ref{sec:online-solution}, we use CBFs / HOCBFs to account for all the rules in the {\em second step} (Sec. \ref{sec:online-opt-ctrl}). In short, we use the optimal $\bm u_{mpc}(k), k\in\{0,\dots, H-1\}$ obtained by solving (\ref{eqn:mpc}) as a reference for the convex CBF controller.
In this two-step approach, the computation complexity is significantly reduced, and the MPC controller leads to better tracking and cost optimization. 

\subsection{TORQ with Rule Activation and Deactivation}
\label{sec:torq_act_deact}

To provide a solution to the online version of Problem \ref{prob:main}, we first
show how to manage TORQ $\langle R,\sim_p,\leq_p\rangle$ in an online fashion. To this goal, we classify the rules from $R$ into {\em instance-dependent} (such as clearance with pedestrian, clearance with parked) and {\em instance-independent} rules (such as speed limit and comfort). Instance-independent rules should always be taken into account. However, instance-dependent rules should only be considered when the corresponding instances are in ego's local sensing range. 

To capture this, we define an online TORQ (oTORQ). Starting from the (full) TORQ defined above, at time $t=0$, we first 
deactivate all the instance-dependent rules. As instances are detected by the sensors, the corresponding instance-dependent rules are activated. Therefore, at each time, we have a TORQ with activated and deactivated rules. The oTORQ is obtained by deleting the deactivated rules. An equivalence class is removed if it is empty. Note that, by the connexity property of the total order, the deletion does not affect the comparability of the remaining rules. In a graphical representation of a oTORQ, if the empty equivalence class is at the top or at the bottom, the incoming or outgoing edges are deleted; otherwise the incoming and outgoing edges into the empty class are collapsed into one edge with the corresponding orientation (see Fig. \ref{fig:otorq}). 

\begin{figure}[!bht]
	\centering
	\includegraphics[scale=0.28]{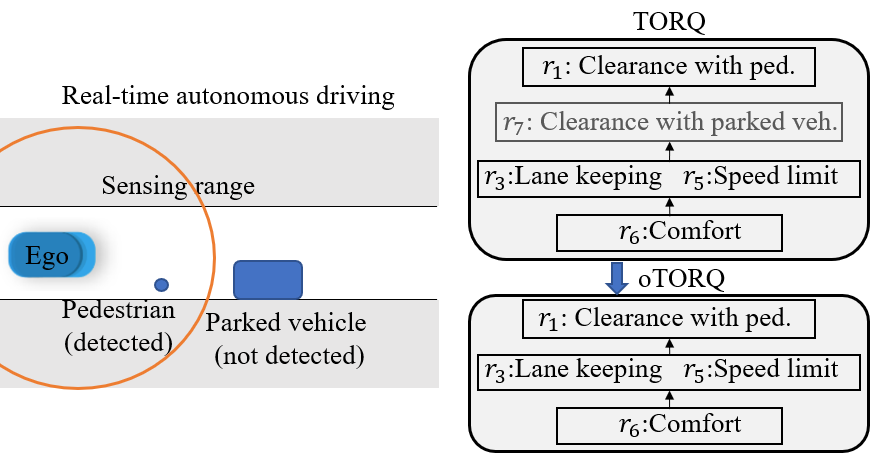}
	\vspace{-3pt}
	\caption{TORQ with rule activation and deactivation and oTORQ; $r_1, r_7$ are the instance-dependent rules. 
	}
	\label{fig:otorq}%
\end{figure}

\subsection{Online Optimal Control}
\label{sec:online-opt-ctrl}

In this section, we present the algorithm providing a solution to the online version of Problem \ref{prob:main}. 
Similar to the offline case to from Sec. \ref{sec:optim-rule-approx}, we use the optimal disk coverage approach from Sec. \ref{sec:rule-approx}
to approximately convert the satisfaction of clearance rules to continuously differentiable constraints.

Let $S_r(t)\subseteq S = 2^{R_\mathcal{O}}$ denote the set of sets of equivalence classes in which all the rules are relaxed at time $t\geq 0$. We initialize $S_r(0)=\emptyset$. At time $t$, we first solve the MPC tracking problem from Eqn. (\ref{eqn:mpc}), and get $\bm u_{mpc}(k), k\in\{0,\dots, H-1\}$. Then, based on information from the local sensing range, we find the oTORQ as described in Sec. \ref{sec:torq_act_deact}. We relax all the rules from the equivalence classes in $S_r(t)$ with the CBF method, and formulate the following optimal control problem with receding horizon $H_t \leq H$, where $H_t$ denotes the forward feasibility horizon for the below QPs, and it is chosen as large as the computational resources allow for.
\begin{equation}
\min_{\bm u, {\delta_i}_{, i\in \mathcal{O}, \mathcal{O}\in S_r(t)}}\!\!\! ||\bm u(k) - \bm u_{mpc}(k)||^2 +\!\!\!\sum_{i\in \mathcal{O}, \mathcal{O}\in S_r(t)}\!\!\!\!\!p_i \delta_i^2( k) \label{eqn:cost_online}
\end{equation}
subject to:

  control bounds (\ref{eqn:control}), HOCBFs for state limitations (\ref{eqn:optim-state-cons}),
{\small\begin{align}
&\begin{aligned}L_{f}^{m_{j}}b_{j}(\bm x)+L_{g}L_{f}^{m_{j}-1}b_{j}(\bm x)\bm
u&+S(b_{j}(\bm x))\\+\alpha_{m_j}(\psi_{m_{j}-1}(\bm x))&\geq0, \forall j\in \mathcal{O}, \forall \mathcal{O}\in R_{\mathcal{O}}\setminus S_r(t),\end{aligned}\label{eqn:optim-not-relax-rules-online}
\\
&\begin{aligned}L_{f}^{m_{i}}b_{i}(\bm x)+L_{g}L_{f}^{m_{i}-1}b_{i}(\bm x)\bm
u+&S(b_{i}(\bm x))\\+\alpha_{m_i}(\psi_{m_{i}-1}(\bm x))&\geq \delta_i,\forall i\in \mathcal{O}, \forall \mathcal{O}\in S_r(t),\end{aligned}\label{eqn:optim-relax-rules-online}
\end{align}}where $k\in\{0,\dots, H_t - 1\}$. The way to solve the above optimization problem is introduced next. 


At time $t$, we get $R$, $R_{\mathcal{O}}$ from the oTORQ. The sensors can only obtain the state information of all the detected instances at the current time $t$. However, in order to solve the above QP (\ref{eqn:cost_online}) for all $k\in\{0,\dots, H_t - 1\}$, we need to know the state information of all the detected instances within horizon $H_t$. In this paper, we get (predict) their states within horizon $H_t$ based on their dynamics (such as (\ref{eqn:predictive})) and their current controls (from the sensor). We may also use machine learning techniques (such as Recurrent Neural Network (RNN)) to estimate their states within horizon $H_t$. 
We solve (\ref{eqn:cost_online}) and obtain $\bm u^*(k)$ at each $k\in\{0,\dots, H_t - 1\}$, and then update the dynamics (\ref{eqn:affine}) with $\bm u^*(k)$ in the following time interval $(t+k\Delta T, t+(k+1)\Delta T]$ and get $\bm x(t + (k+1)\Delta T)$ for the next QP (\ref{eqn:cost_online}) at step $k + 1$. If the optimization problem from Eqn. (\ref{eqn:cost_online}) is infeasible for any $k\in\{0,\dots, H_t - 1\}$, then we add the equivalence class that has higher priority than the existing ones in $S_r(t)$ to the set $S_r(t)$, and solve (\ref{eqn:cost_online}) for all $k\in\{0,\dots, H_t - 1\}$ again. This process is repeated until problem (\ref{eqn:cost_online}) is feasible for all $k\in\{0,\dots, H_t - 1\}$.

Note that we do not consider the inter-sampling effect (i.e., constraint satisfaction within each time interval)
in the CBF-based approach for both online and offline control in this paper. This can be addressed using the method from \cite{Xiao2021CDC}. 
As an alternative, in this paper, we check the satisfaction of all the constraints corresponding to the rule $j\in \mathcal{O}, \forall \mathcal{O}\in R_{\mathcal{O}}\setminus S_r(t)$ through a receding horizon approach introduced below. If all the constraints from rule $ j\in \mathcal{O}, \forall \mathcal{O}\in R_{\mathcal{O}}\setminus S_r(t)$ are satisfied for all $k\in\{0,\dots, H_t - 1\}$, then we set $S_r(t)$ to empty. Otherwise, we find the highest priority level equivalence class in which a corresponding constraint is violated. We set $S_r(t)$ to include all the equivalent classes that are with lower priority than this equivalent class. In this way, this violated constraint is possible to be satisfied in the following time. Eventually, all the constraints corresponding to the rule $j\in \mathcal{O}, \forall \mathcal{O}\in R_{\mathcal{O}}$ will be satisfied. 

After we determine which rules (i.e., $S_r(t)$) need to be relaxed through the receding horizon approach introduced above, we apply the control from solving (\ref{eqn:cost_online}) with $k = 0$ only one time step forward at the current time. This is due to the fact that the state estimation errors of the instances and the error of the predictive model (\ref{eqn:predictive}) with respect to (\ref{eqn:affine}) increases as $k$ increases.  
 The above process is repeated for all the following times. The online rule-based control algorithm is summarized in Alg. \ref{alg:online}.

\begin{algorithm}
\caption{Online optimal control} \label{alg:online}
	\KwIn{ Sensor information, State information for all the instances in the horizon $H_t$, System (\ref{eqn:affine}) with $\bm x(0)$, cost function (\ref{eqn:gcost}), control bound (\ref{eqn:control}), state constraint (\ref{eqn:state}), TORQ $\langle R,\sim_p,\leq_p\rangle$, reference trajectory $\mathcal{X}_r$}
	\KwOut{Online optimal control at each time step}
	$S_r(0) = \{\emptyset\}$\;
	\While{$t\leq T$
	}{
    	1. Solve (\ref{eqn:mpc}) and get $\bm u_{mpc}(k), k\in\{0,\dots,H\}$ \;
    	Find oTORQ as in Sec. \ref{sec:torq_act_deact} based on sensor detection\;
    	2. Get $R$ and $R_{\mathcal{O}}$ from the oTORQ\;
    	3. Get the states of all detected instances within horizon $H_t\leq H$ in $R$\;
    	4. Solve (\ref{eqn:cost_online}) for all $k\in\{0,\dots, H_t\}$ and update (\ref{eqn:affine}) in the corresponding time interval\;
    	\While{$\exists k\in\{0,\dots, H_t\}$ such that (\ref{eqn:cost_online}) is infeasible }{
    	    Add the higher priority level equivalent class to $S_r(t)$\;
        	Solve (\ref{eqn:cost_online}) for all $k\in\{0,\dots, H_t\}$ and update (\ref{eqn:affine}) in the corresponding time interval\;
    	}
    	\eIf{All the constraints from the rule $ j\in \mathcal{O}, \forall \mathcal{O}\in R_{\mathcal{O}}\setminus S_r(t)$ are satisfied for all $k\in\{0,\dots, H_t - 1\}$}
    	{
    	Set $S_r(t) = \{\emptyset\}$\;
    	}
    	{
    	Find the highest-priority-level equivalent class in which a corresponding constraint is violated\;
    	Set $S_r(t)$ to include all the equivalent classes that are with lower priority than this equivalent class\;
    	}
    	5. Solve (\ref{eqn:cost_online}) with $k = 0$ and get $\bm u^*(0)$\;
    	6. Apply $\bm u^*(0)$ at the current time, $t \leftarrow t+\Delta t$\;
	}
\end{algorithm}

\begin{remark}
(Complexity) The complexity of Alg. \ref{alg:online} at each time depends on the horizons $H$ and $H_t$. At each time step $k$, we have $H_t\zeta + 1$ QPs (\ref{eqn:cost_online}) and a NLP (\ref{eqn:mpc}), where $\zeta \in\mathbb{N}$ denotes the number of iterations needed to find a feasible solution for all the $H_t$ QPs (the inner loop in Alg. \ref{alg:online}). The time complexity of the NLP depends on $H$, the cost in (\ref{eqn:mpc}) and the predictive model (\ref{eqn:predictive}). 
\end{remark}

\section{Case Study}
\label{sec:case}

In this section, we apply the methodology developed in this paper to specific vehicle dynamics and various driving scenarios. All the results in this section were produced using a user-friendly tool that allows to create or load maps, place vehicles and pedestrians on the map, and specify ego dynamics. Details on the implementation, including running times are given in Appendix \ref{sec:tool}. 

Ego dynamics \eqref{eqn:affine} are defined with respect to a reference trajectory \cite{Rucco2015}, which measures the along-trajectory distance $s\in\mathbb{R}$ and the lateral distance $d\in\mathbb{R}$ of the vehicle Center of Gravity (CoG) with respect to the closest point on the reference trajectory as follows: 
{\small\begin{equation} \label{eqn:vehicle}
   \underbrace{\left[
\begin{array}[c]{c}
    \dot s\\
    \dot d\\
    \dot \mu\\
    \dot v\\
    \dot a\\
    \dot \delta\\
    \dot \omega
\end{array}
\right]}_{\dot {\bm x}}
=
\underbrace{\left[
\begin{array}
[c]{c}%
    \frac{v\cos(\mu + \beta)}{1 - d\kappa}\\
    v\sin(\mu + \beta)\\
    \frac{v}{l_r}\sin\beta - \kappa\frac{v\cos(\mu + \beta)}{1 - d\kappa}\\
    a\\
    0\\
    \omega\\
    0
\end{array}
\right]}_{f(\bm x)}
+
\underbrace{\left[
\begin{array}[c]{cc}%
    0 & 0\\
    0 & 0\\
    0 & 0\\
    0 & 0\\
    1 & 0\\
    0 & 0\\
    0 & 1
\end{array}
\right]}_{g(\bm x)}
\underbrace{\left[
\begin{array}[c]{c}%
    u_{jerk}\\
    u_{steer}
\end{array}
\right]}_{\bm u},
\vspace{-2pt}
\end{equation}
}where $\mu$ is the vehicle local heading error determined by the difference of the global vehicle heading $\theta\in\mathbb{R}$ in (\ref{eqn:center}) and the tangent angle $\phi\in\mathbb{R}$ of the closest point on the reference trajectory (i.e., $\theta = \phi + \mu$); $v$, $a$ denote the vehicle linear speed and acceleration; $\delta$, $\omega$ denote the steering angle and steering rate, respectively; $\kappa$ is the curvature of the reference trajectory at the closest point; $l_r$ is the length of the vehicle from the tail to the CoG; and $u_{jerk}$, $u_{steer}$ denote the two control inputs for jerk and steering acceleration as shown in Fig. \ref{fig:frame}. $\beta = \arctan\left(\frac{l_r}{l_r + l_f}\tan\delta\right)$ where $l_f$ is the length of the vehicle from the head to the CoG.\vspace{-2pt}
\begin{figure}[thpb]
	\centering
	\floatbox[{\capbeside\thisfloatsetup{capbesideposition={right,top},capbesidewidth=3cm}}]{figure}[\FBwidth]{$\hspace{-8mm}$\caption{Coordinates of ego w.r.t a reference trajectory.}
	\label{fig:frame}}
	{\includegraphics[scale=0.3]{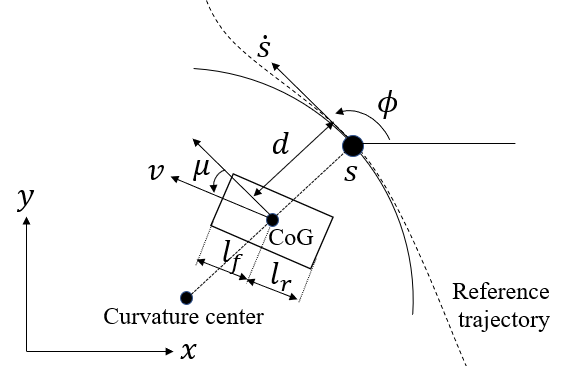}}
\end{figure}

We consider the cost function in \eqref{eqn:cost2} as:
\begin{equation}
     \min_{u_{jerk}(t), u_{steer}(t)}\int_{0}^{T}\left[u_{jerk}^2(t) + u_{steer}^2(t)\right]dt.
\end{equation}

The reference trajectory $\mathcal{X}_r$ is the middle of ego's current lane, and is assumed to be given as an ordered sequence of points $\bm p_1$, $\bm p_2$, $\dots$, $\bm p_{N_r}$, where $\bm p_i \in \mathbb{R}^2, i=1,\dots,N_r$ ($N_r$ denotes the number of points). We can find the reference point $p_{i(t)}$, $i:[0,T]\rightarrow \{1,\ldots,N_r\}$ at time $t$ as:
\vspace{-4pt}
\begin{equation}\label{eqn:tracking}
\begin{aligned}
    i(t)= \begin{cases} i(t) + 1,   ||\bm p(t) - \bm p_{i(t)}||\leq \gamma,\\
     j,  \exists j\in\{1,2,\dots, N_r\}:||\bm p(t) \!-\! \bm p_{i(t)}||\!\geq\! ||\bm p(t) \!-\! \bm p_{j}||,
\end{cases}
\end{aligned}
\end{equation}
where $\bm p(t)\in \mathbb{R}^2$ denotes ego's  location. $\gamma > 0$, and $i(0) = k$ for a $k\in\{1,2,\dots, N_r\}$ is chosen such that $||\bm p(0) - \bm p_{j}||\geq ||\bm p(0) - \bm p_{k}|, \forall j\in\{1,2,\, N_r\}$. Once we get $\bm p_{i(t)}$, we can update the progress $s$, the error states $d,\mu$ and the curvature $\kappa$ in (\ref{eqn:vehicle}).  

For offline control, trajectory tracking means stabilizing the error states $d, \mu$ ($\bm y = (d,\mu)$ in (\ref{eqn:clf1})) to 0, as introduced in Sec. \ref{sec:tracking}. We also wish ego to achieve a desired speed $v_d > 0$ (otherwise, ego may stop in curved lanes). We achieve this by re-defining the CLF $V(\bm x)$ in  (\ref{eqn:clf1}) as $V(\bm x) = ||\bm y||^2 + c_0(v-v_d)^2, c_0 > 0$.  As the relative degree of $V(\bm x)$ w.r.t. (\ref{eqn:vehicle}) is larger than 1, as mentioned in Sec. \ref{sec:tracking}, we use input-to-state linearization and state feedback control \cite{Khalil2002} to reduce the relative degree to one \cite{Xiao2020}. For example, for the desired speed part in the CLF $V(\bm x)$ ( (\ref{eqn:vehicle}) is in linear form from $v$ to $u_{jerk}$, so we don't need to do linearization), we can find a desired state feedback acceleration $\hat a = -k_1(v - v_d), k_1 > 0$. Then we can define a new CLF in the form $V(\bm x) = ||\bm y||^2 + c_0(a -\hat a)^2 = ||\bm y||^2 + c_0(a + k_1(v - v_d))^2$ whose relative degree is just one w.r.t. $u_{jerk}$ in (\ref{eqn:vehicle}). We proceed similarly for driving $d, \mu$ to 0 in the CLF $V(\bm x)$ as the relative degrees of $d, \mu$ are also larger than one.

To provide a fair comparison with the offline case, 
we choose the same discretization time $\Delta T$. We follow Sec. \ref{sec:online_track}, and get the following predictive model with the ADM method (under the first order approximation): 
{\small\begin{equation} \label{eqn:vehicle_mpc}
   \left[\!\!\!\!
\begin{array}[c]{c}
    s(k\!+\!1)\\
    d(k\!+\!1)\\
    \mu(k\!+\!1)\\
    v(k\!+\!1)\\
    a(k\!+\!1)\\
    \delta(k\!+\!1)\\
    \omega(k\!+\!1)
\end{array}
\!\!\!\!\right]
\!\!=\!\!
\left[\!\!\!
\begin{array}
[c]{c}%
    s(k) + \frac{v(k)\cos(\mu(k) + \beta(k))}{1 - d(k)\kappa(k)}\Delta T\\
    d(k) + v(k)\sin(\mu(k) + \beta(k))\Delta T\\
    \mu(k)\! +\! (\frac{v(k)}{l_r}\sin\beta(k) \!-\! \kappa(k)\frac{v(k)\cos(\mu(k) \!+\! \beta(k))}{1 \!-\! d(k)\kappa(k)})\Delta T\\
    v(k) + a(k)\Delta T + \frac{1}{2}u_{jerk}(k)\Delta T^2\\
    a(k) + u_{jerk}(k)\Delta T\\
    \delta(k) + \omega(k)\Delta T + \frac{1}{2}u_{steer}(k)\Delta T^2 \\
    \omega(k) + u_{steer}(k)\Delta T
\end{array}
\!\!\!\right]
\vspace{-2pt}
\end{equation}
}where $k = 0, 1, 2,\dots,$ and $\kappa(k)$ is obtained through regression of the reference points  from $\mathcal{X}_r$ (usually the lane center line) using splines.

The control bounds (\ref{eqn:control}) and state constraints (\ref{eqn:state}) are given by:
\begin{equation}\label{eqn:physical}
    \begin{aligned}
    \text{speed constraint: }&  v_{\min} \leq v(t)\leq v_{\max},\\
    \text{acceleration constraint: }&  a_{\min}\leq a(t)\leq a_{\max},\\
    \text{jerk control constraint: }& u_{j,\min}\leq u_{jerk}(t)\leq u_{j,\max},\\
    \text{steering angle constraint: }& \delta_{\min}\leq \delta(t)\leq \delta_{\max},\\
    \text{steering rate constraint: }& \omega_{\min}\leq \omega(t)\leq \omega_{\max},\\
    \text{steering control constraint: }& u_{s,\min}\leq u_{steer}(t)\leq u_{s,\max},
    \end{aligned}
\end{equation}

We consider the TORQ $\langle R,\sim_p,\leq_p\rangle$ from Fig. \ref{fig:case_rb}, with rules $R = \{r_1, r_2, r_3, r_4, r_5, r_6, r_7, r_8\}$, where $r_1$ is a pedestrian clearance rule; $r_2$ and $r_3$ are clearance rules for staying in the drivable area and lane, respectively; $r_4$ and $r_5$ are non-clearance rules specifying maximum and minimum speed limits, respectively; $r_6$ is a comfort non-clearance rule; and $r_7$ and $r_8$ are clearance rules for parked and moving vehicles, respectively. The formal rule definitions (statements, violation metrics) are given in Appendix \ref{sec:app-rule-def}. These metrics are used to compute the scores for all the trajectories in the three scenarios below. 
The optimal disk coverage from Sec. \ref{sec:rule-approx} is used to compute the optimal controls for all the clearance rules, which are implemented using HOCBFs.
 \begin{figure}[thpb]
	\centering
	\floatbox[{\capbeside\thisfloatsetup{capbesideposition={right,top},capbesidewidth=3cm}}]{figure}[\FBwidth]{\caption{TORQ priority structure for case studies. 
	}
	\label{fig:case_rb}}
	{\includegraphics[scale=0.25]{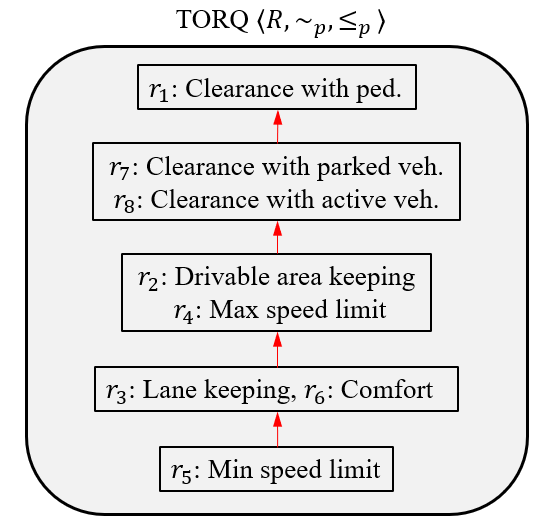}}
		\vspace{-3mm}
\end{figure}

In the following, we consider three common driving scenarios. For each of them, we solve the optimal control Problem \ref{prob:main} with both the online and offline methods, and perform pass/fail evaluation. In all three scenarios, in the pass/fail evaluation, an initial candidate trajectory is drawn ``by hand" using the tool described in Appendix \ref{sec:tool}. For offline control, we use CLFs to generate a feasible trajectory $\mathcal{X}_c$ that tracks the candidate trajectory subject to the vehicle dynamics (\ref{eqn:affine}), control bounds (\ref{eqn:control}) and state constraints (\ref{eqn:state}). For the online control method, we assume the sensor sensing range is $20m$ and set the horizon $H = H_t = 10$s.

\subsection{Tracking comparison}
\label{sec:track-comp}

We first compare the tracking performances of the CLF-based offline and MPC-based online controllers proposed above, and discuss their advantages and disadvantages. We consider a curved road with a single lane as shown in Fig. \ref{fig:traj}. The CLF method is implemented as a sequence of QPs (one at each time step, solved using the quadprog solver in Matlab). The MPC method is reduced to a sequence of nonlinear programs (solved using the fmincon solver in Matlab). 

\begin{figure}[thpb]
	\centering
	\vspace{-3mm}
	\includegraphics[scale=0.45]{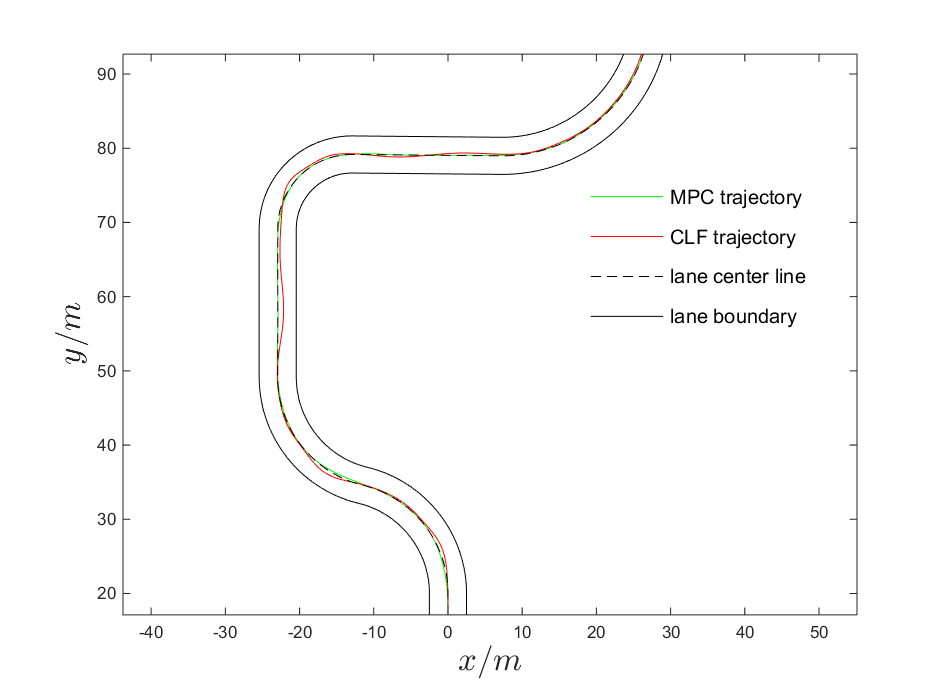}
	\vspace{-3mm}
	\caption{Tracking trajectory comparison between the CLF-based offline and MPC-based online methods.
	}
	\label{fig:traj}%
	\vspace{-3mm}
\end{figure}

\begin{figure}[thpb]
	\centering
	\vspace{0mm}
	\includegraphics[scale=0.6]{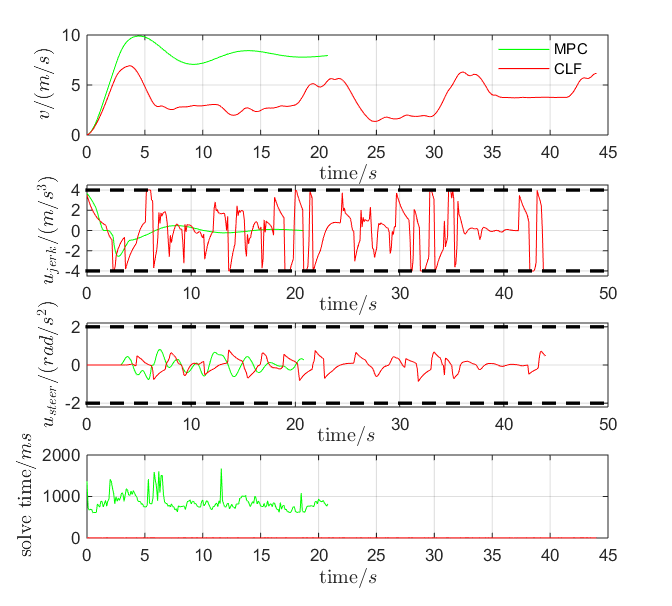}
	\vspace{-3mm}
	\caption{Speed, control and computation time profiles for the CLF and MPC methods. The base time for the MPC curves is shorter as tracking finishes earlier (the ego speed is higher).
	}
	\label{fig:track_state}%
	\vspace{-3mm}
\end{figure}

As it can be seen in Fig. \ref{fig:traj}, the MPC controller can track the lane center line with less error than the CLF controller. This makes sense, as the MPC method uses a receding horizon, while the CLF method enforces converge during each (short) time interval. Moreover, the MPC method produces a trajectory with a higher speed than the CLF method (see Fig. \ref{fig:track_state}). The main disadvantage of the MPC method is that it is more difficult to solve than the CLF method, and has much higher computation time
(between $1s$ and $2$) than the CLF method (about $4ms$), as shown in the last frame of Fig. \ref{fig:track_state} (the CLF times are not distinguishable on the flat red line). Note that we can significantly reduce the computation time for the MPC method if we use more powerful solvers, such as FORCES (Embotech).

\subsection{Scenario 1}


In this scenario, the traffic instances are: (1) an active vehicle, which moves in the same direction as ego in an adjacent lane, (2) a parked vehicle, which takes some space in the ego's lane, and (3) a static pedestrian (see Fig. \ref{fig:case12}). 

\textbf{Offline optimal control:} 
We solve the optimal control problem (\ref{eqn:cost2}) by starting the rule relaxation with $S_1=\{\emptyset\}$ (i.e., without relaxing any rule). This problem is infeasible in the given scenario since ego could not maintain the required distance between both the active and the parked vehicles as the clearance rules are speed-dependent, and $r_5$ imposes a minimum speed limit. Therefore, we relaxe the next lowest priority equivalence class set in $S_{sorted}$, i.e., the minimum speed limit rule in $S_2=\{\{r_{5}\}\}$. With this relaxation, we are able to find a feasible trajectory as illustrated in Fig. \ref{fig:case12} (left). 
By checking $\delta_i$ for $r_5$ from \eqref{eqn:cost2}, we find that it was positive in some time intervals in $[0,T]$, and thus, $r_5$ is indeed relaxed. The total violation score for rule $r_{5}$ for the generated trajectory is 0.539, and all the other rules in $R$ are satisfied. Therefore, by Def. \ref{def:rb_satisfy}, the generated trajectory satisfies the TORQ $\langle R,\sim_p,\leq_p\rangle$ from Fig. \ref{fig:case_rb}.
 
 \begin{figure}[thpb]
	\centering
	\vspace{-3mm}
$\hspace{-6mm}$\includegraphics[scale=0.5]{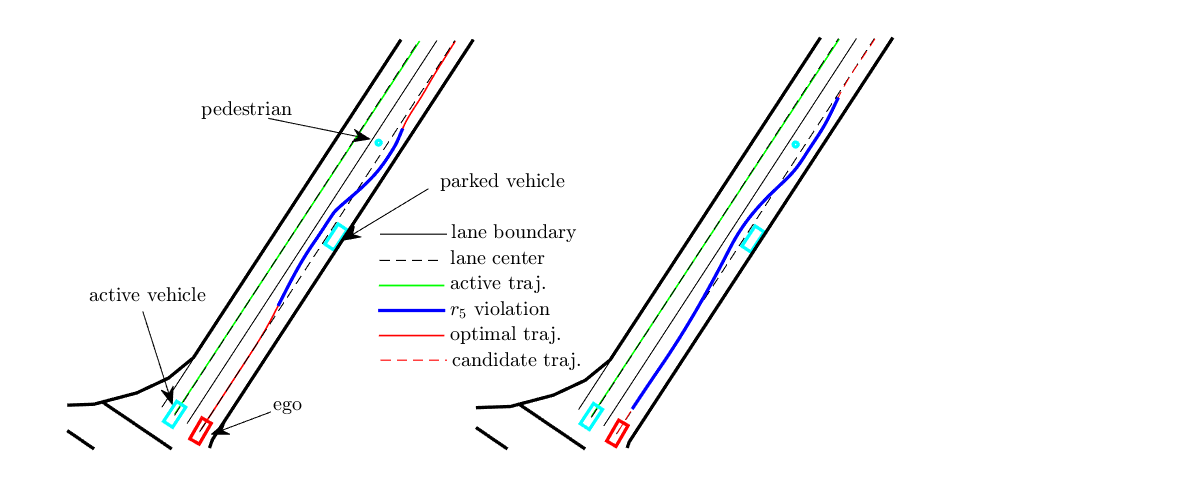}
	\vspace{-3mm}
	\caption{Offline optimal control (left) and Pass/Fail (right) for Scenario 1: the subsets of offline optimal and candidate trajectories violating $r_5$ are shown in blue. The alternative trajectory for Pass/Fail (right) is the same as the offline optimal control (left).
	}
	\label{fig:case12}%
	\vspace{-3mm}
\end{figure}
 

\textbf{Pass/Fail:} The candidate trajectory $\mathcal{X}_c$ is shown in Fig. \ref{fig:case12} (right). It only violates rule $r_5$ with a total violation score of  0.682. Following the method from Sec. \ref{sec:p/f}, a better trajectory would either not violate any rules or violate only $r_5$. As shown in the above off-line optimal control problem for this scenario, we cannot find a feasible solution if we do not relax $r_5$. Since the violation of $r_5$ by the candidate trajectory is larger than 0.539 (the score of the optimal trajectory from Fig. \ref{fig:case12} (left)), we fail the candidate trajectory. 

\textbf{Online vs. offline optimal control:} We first consider online optimal control for exactly the same setting as in the offline case (Fig. \ref{fig:case12} (left)). We assume that ego only has local information about traffic instances in a disk centered at the ego and with radius. 
If we require that ego follow the TORQ in Fig. \ref{fig:case_rb}, then it stops
before the parked vehicle, because $r_3$ (lane keeping) has higher priority than $r_5$ (minimum speed limit). Thus, ego would violate $r_5$ first, which leads it to stop.  
This shows that the online control method is more conservative than the offline one. If we relax $r_3$, which does not follow the TORQ and we do not do it here, it can be shown that the problem is feasible (this will be shown in Case 3). In order to have a fair comparison between the offline and the online cases, and make ego pass the parked vehicle in the online control case, we set the location of the parked vehicle a little off the lane center line, as shown in Fig. \ref{fig:case1_compare}. Note that this would not be necessary if we had a ``reach goal" rule. 
Although the trajectories from the online and offline control are close to each other, the violation metric  of $r_5$ for the online control (0.6219) is worse than the offline one (0.4410). This further shows that the online method is more conservative as it only uses local information, which results in lower moving speed. In other words, with more information, we can find a better solution in the offline case. 
\begin{figure}[thpb]
	\centering
	$\hspace{-4mm}$\includegraphics[scale=0.45]{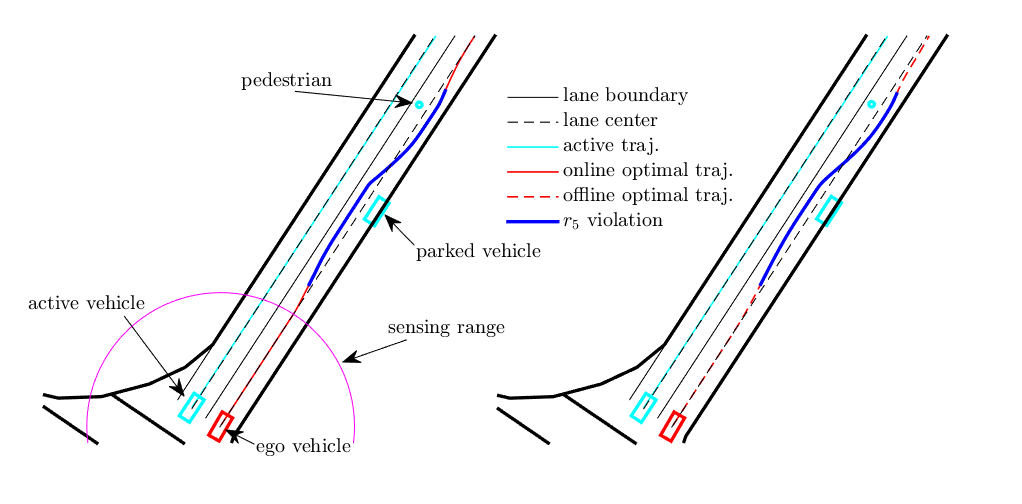}
	\vspace{-5mm}
	\caption{Online (left) vs. offline (right) control for Scenario 1: the subsets of online and offline optimal trajectories violating $r_5$ are shown in blue.}
	\label{fig:case1_compare}
	\vspace{-3mm}
\end{figure}

\subsection{Scenario 2}

In this scenario, the traffic instances are: (1) one active vehicle, which moves in the same direction with ego in an adjacent lane, (2) two parked vehicles, and (3) two static pedestrians (see Fig. \ref{fig:case2}). 

\textbf{Offline optimal control:} Similar to Scenario 1, the optimal control problem (\ref{eqn:cost2}) starting with $S_1=\{\emptyset\}$ (without relaxing any rules in $R$) is infeasible. We relax the lowest priority rule set in $S_{sorted}$, i.e., the minimum speed rule in $S_2=\{\{r_{5}\}\}$, for which we are able to find a feasible trajectory as illustrated in Fig. \ref{fig:case2}. Again, the $\delta_i$ for $r_5$ is positive in some time intervals in $[0,T]$, and thus, $r_5$ is indeed relaxed. The total violation score of $r_{5}$ for the generated trajectory is 0.646, and all the other rules in $R$ are satisfied. 

\begin{figure}[thpb]
	\centering
	\vspace{1mm}
	\floatbox[{\capbeside\thisfloatsetup{capbesideposition={right,top},capbesidewidth=3cm}}]{figure}[\FBwidth]{$\hspace{-9mm}$\vspace{-4mm}\caption{Offline optimal control for Scenario 2: the subset of the optimal ego trajectory violating $r_5$ is shown in blue.}
	\label{fig:case2}}
	{$\hspace{-9mm}$\includegraphics[scale=0.5]{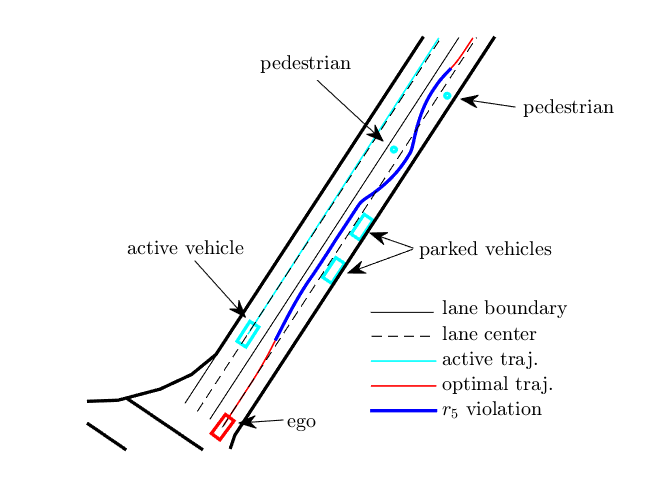}}
	\vspace{-4mm}
\end{figure}

\textbf{Pass/Fail:} The candidate trajectory $\mathcal{X}_c$ shown in Fig. \ref{fig:case2_pf} (left) violates rules $r_1$ (clearance with pedestrian), $r_{3}$ (lane keeping), and $r_{8}$ (clearance with active vehicles) with total violation scores of 0.01, 0.23, and 0.22 (see \eqref{eqn:r_1} for the definition of violation score for $r_1$ and \cite{Xiao2021ICCPS} for the other rules). 

In this scenario, if we followed the procedure from Sec. \ref{sec:p/f}, we would start by relaxing $r_1$, which may lead to a better trajectory than the candidate one (starting from the highest priority equivalence class in $S_{sorted}$ in Alg. \ref{alg:sort} in order to quickly get a solution), which would not give us a desirable trajectory as it could still violate the highest priority rule $r_1$. In order to obtain a desirable trajectory, we show the iteration of relaxing the rules in the equivalence classes $\mathcal{O}_2=\{r_{3}, r_6\}$ and $\mathcal{O}_1=\{r_{5}\}$ to find a feasible trajectory that is also better than the candidate one.
The optimal control problem (\ref{eqn:cost2}) generates the red-solid curve shown in Fig. \ref{fig:case2_pf}, and only the $\delta_6$ (the relaxation for $r_6$) is 0 for all $[0,T]$. Thus, $r_6$ does not need to be relaxed. The generated trajectory violates rules $r_{3}$ and $r_{5}$ with total violation scores 0.124 and 0.111, respectively, but satisfies all the other rules including the highest priority rule $r_1$. According to Def. \ref{def:traj_cmp} for the given $\langle R,\sim_p,\leq_p\rangle$ in Fig. \ref{fig:case_rb}, the new generated trajectory is better than the candidate one, and we fail the candidate trajectory. Note that although this trajectory violates the lane keeping rule, it has a smaller violation score for $r_{5}$ compared to the trajectory obtained from the optimal control in Fig. \ref{fig:case2} (0.111 v.s. 0.646), i.e., the average speed of ego in the red-solid trajectory in Fig. \ref{fig:case2_pf} is larger.
\begin{figure}[thpb]
\centering
	\vspace{-3mm}
    $\hspace{-2mm}$\includegraphics[scale=0.45]{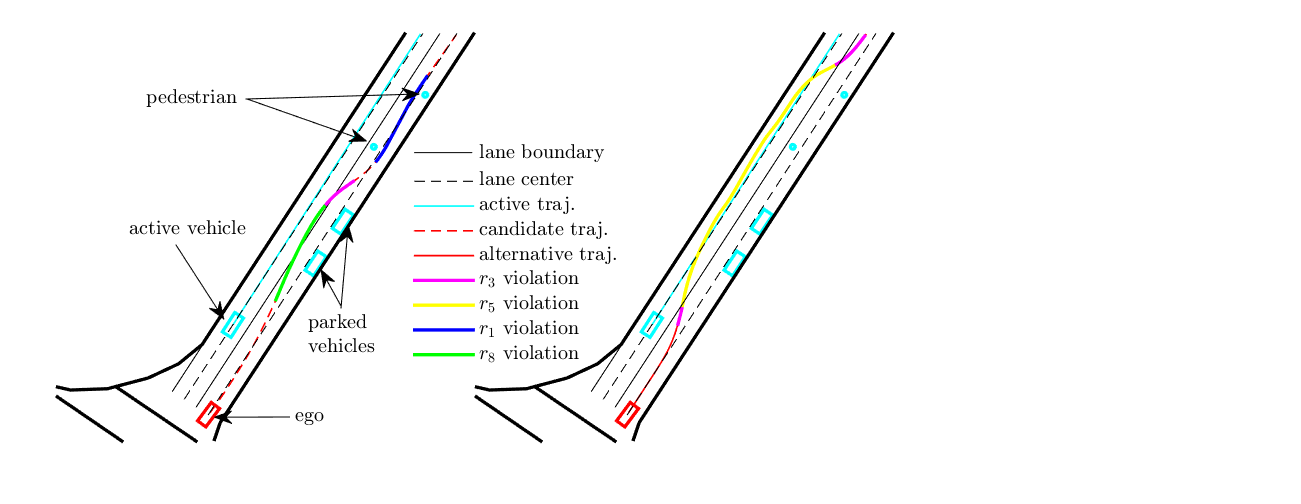}
	\vspace{-6mm}
	\caption{Pass/Fail for Scenario 2:  the subsets of the candidate trajectory (left) violating $r_8, r_3, r_1$ are shown in green, magenta and blue, respectively; the subsets of the alternative trajectory (right) violating $r_5$, and $r_3$ are shown in yellow and magenta, respectively.} 
	\label{fig:case2_pf}
	\vspace{-3mm}
\end{figure}

\textbf{Online vs. offline optimal control:} 
 Similar to Scenario 1, in order to make the ego pass the parked vehicles in the online control case, we set the locations of the parked vehicles a little off the lane center line, as shown in Fig. \ref{fig:case2_compare}. Although the trajectories from the online and offline control are close to each other, the total violation metric for $r_5$ for the online case (0.7221) is worse than the offline one (0.5280). Again, this shows that the online method is more conservative as it only has local information. 
\begin{figure}[thpb]
	\centering
	$\hspace{-4mm}$\includegraphics[scale=0.45]{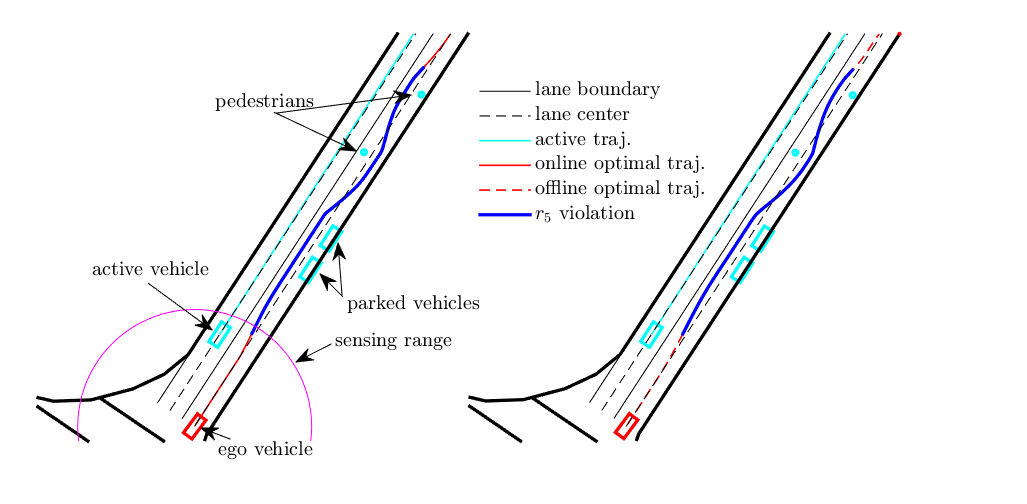}
	\vspace{-5mm}
	\caption{Online (left) vs. offline control for Scenario 2: the subsets of online and offline optimal trajectories violating $r_5$ are shown in blue.}
	\label{fig:case2_compare}
	\vspace{-3mm}
\end{figure}

\subsection{Scenario 3}

In this scenario, the traffic instances are: (1) one active vehicle, which moves in the same direction as ego in an adjacent lane, (2) one parked vehicle, and (3) two static pedestrians, one of which just got out of the parked car (see Fig. \ref{fig:case3}). 

\textbf{Offline optimal control:} Similar to Scenario 1, the optimal control problem (\ref{eqn:cost2}) starting from $S_1=\{\emptyset\}$ (without relaxing any rules in $R$) is infeasible. We relax the lowest priority rule set in $S_{sorted}$, i.e., the minimum speed rule $S_2=\{\{r_{5}\}\}$, and solve the optimal control problem. In the (feasible) generated trajectory, ego stops before the parked vehicle, which satisfies all the rules in $R$ except $r_{5}$. Thus, by Def. \ref{def:rb_satisfy}, the generated trajectory satisfies the TORQ $\langle R,\sim_p,\leq_p\rangle$. However, this is not a desirable behavior.~\footnote{Note that
such an undesirable behaviour would not satisfy the TORQ
if we had a ``reach goal" rule with priority higher than $r_3$.} We further relax the lane keeping $r_{3}$ and comfort $r_6$ rules and find the feasible trajectory shown in Fig. \ref{fig:case3}. $\delta_i$ for $r_6$ is 0 for all $[0,T]$; $r_6$ does not need to be relaxed. The total violation scores for rules $r_{3}$ and $r_{5}$ are 0.058 and 0.359, respectively, and all other rules in $R$ are satisfied.
\begin{figure}[thpb]
  \vspace{3mm}
	\centering
	\floatbox[{\capbeside\thisfloatsetup{capbesideposition={right,top},capbesidewidth=3cm}}]{figure}[\FBwidth]{$\hspace{-9mm}\vspace{-4mm}$\caption{Offline optimal control for Scenario 3: the subsets of optimal ego trajectory violating $r_5$ and $r_3$ are shown in blue and green, respectively.}
	\label{fig:case3}}
	{$\hspace{-9mm}$\includegraphics[scale=0.5]{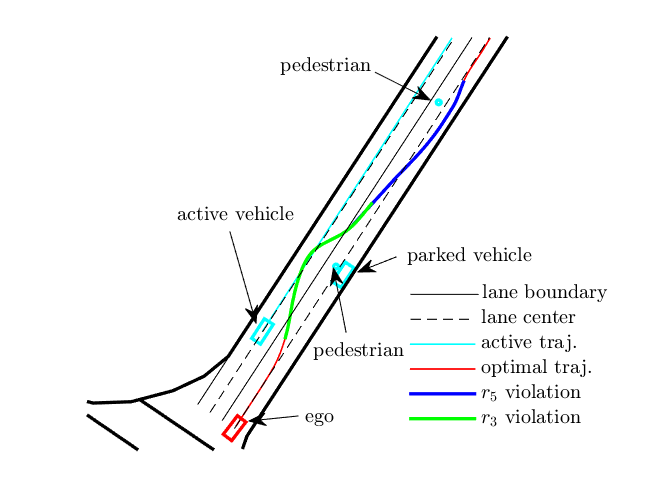}}
	\vspace{-7mm}
\end{figure}

\textbf{Pass/Fail:} The candidate trajectory $\mathcal{X}_c$ shown in Fig. \ref{fig:case3_pf} violates rules $r_{3}$ and $r_{8}$ with total violation scores of 0.025 and 0.01, respectively. In this scenario, from the offline optimal control in Fig. \ref{fig:case3}, we have that ego would stop before the parked vehicle if we followed the procedure from Sec. \ref{sec:p/f}. Thus, in order to show a desirable alternative trajectory, 
we show the iteration of relaxing the rules in the equivalence classes $\mathcal{O}_2=\{r_{3}, r_6\}$ and $\mathcal{O}_1=\{r_{5}\}$ (all have lower priorities than $r_{8}$). The optimal control problem (\ref{eqn:cost2}) generates the red-solid curve shown in Fig. \ref{fig:case3_pf}. By checking $\delta_6$ for $r_6$, we find that $r_6$ is indeed not relaxed. The generated alternative trajectory violates rules $r_{3}$ and $r_{5}$ with total violation scores 0.028 and 0.742, respectively, but satisfies all the other rules including $r_{8}$. According to Def. \ref{def:traj_cmp} and Fig. \ref{fig:case_rb}, the new generated trajectory (although violates $r_{3}$ more than the candidate trajectory, it does not violate $r_{8}$ which has a higher priority) is better than the candidate one. We fail the candidate trajectory. 

\begin{figure}[thpb]
	\centering
$\hspace{-3mm}$\includegraphics[scale=0.485]{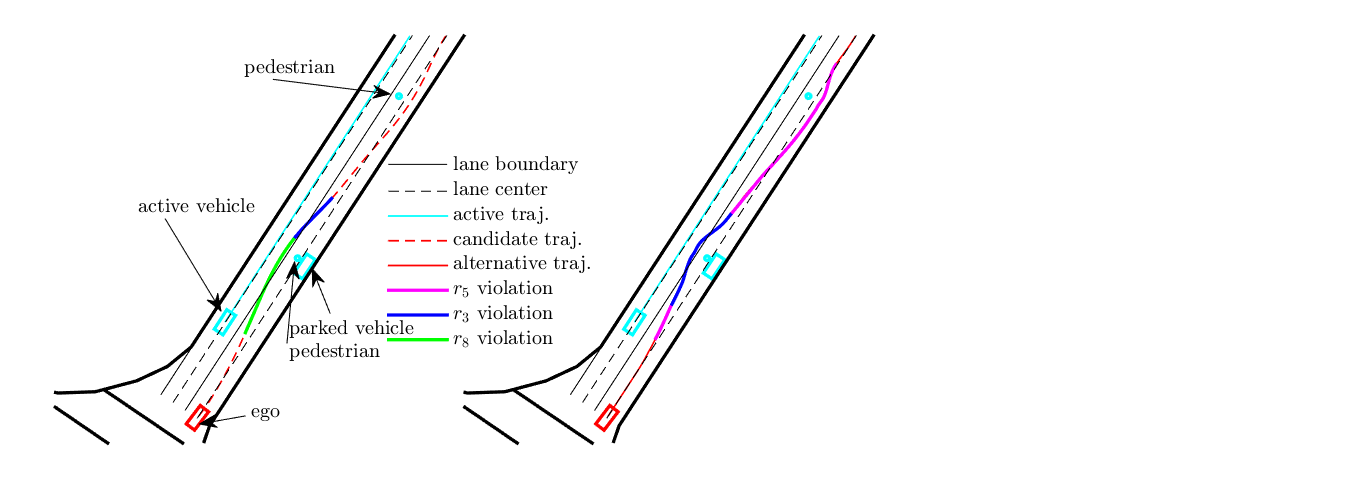}
	\vspace{-6mm}
	\caption{Pass/Fail for Scenario 3: the subsets of the candidate trajectory (left) violating $r_8, r_3$ are shown in green and blue, respectively; the subsets of the alternative trajectory (right) violating $r_5, r_3$ are shown in magenta and blue, respectively. 
	}
	\label{fig:case3_pf}
	\vspace{-3mm}
\end{figure}

\textbf{Online vs. offline optimal control:} 
In this case, the ego can not pass the parked vehicle and the pedestrian near it for both the online and offline methods if it follows the Algs. \ref{alg:sort} or \ref{alg:online} and the TORQ in Fig. \ref{fig:case_rb}. To have a fair comparison between the offline and the online cases and to get desirable trajectories, we relax the lane keeping rule $r_3$ (which does not follow the TORQ in Fig. \ref{fig:case_rb}).
As shown in Fig. \ref{fig:case3_compare}, the trajectory from the online control violates $r_3$ and $r_5$ with violation scores 0.0580 and 0.6034, respectively, while the violation scores of $r_3$ and $r_5$ for the offline control method are 0.0577 and 0.3594, respectively. The offline control method has less violations scores than the online control, which again shows that the online method is more conservative. 

All the three scenarios above show that ego may not have a reasonable driving behavior based on the current TORQ in Fig. \ref{fig:case_rb}. In order to address this problem, we may add more rules. For example, we may define a new reach-goal rule that has higher priority than $r_3$ and has lower priority than $r_2$ and $r_4$. The completeness of a TORQ for a reasonable driving behavior is subject of current research.

\begin{figure}[thpb]
	\centering
	$\hspace{-4mm}$\includegraphics[scale=0.45]{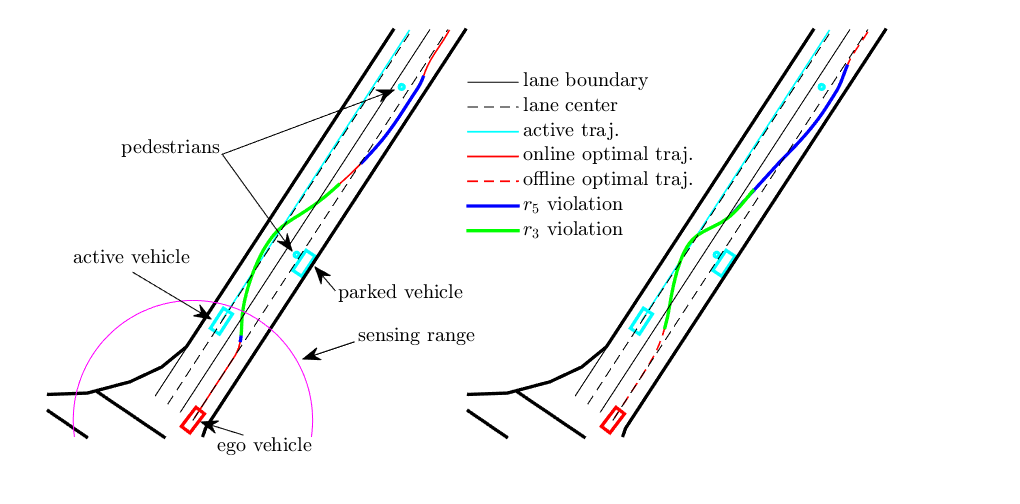}
	\vspace{-5mm}
	\caption{Online versus offline control for Scenario 3: the subsets of online and offline optimal trajectories violating $r_3, r_5$ are shown in green and blue, respectively. }
	\label{fig:case3_compare}
	\vspace{-3mm}
\end{figure}

\section{Conclusions and Future Work}
\label{sec:con}

We developed a framework to design optimal control strategies for autonomous vehicles that aim to satisfy a set of driving rules with a given priority structure, while following a reference trajectory and satisfying control and state limitations. We showed that, for commonly used driving rules, by using control barrier functions and control Lyapunov functions, the problem can be cast as an iteration of optimal control problems, where each iteration involves a sequence of quadratic programs. We proposed offline and online versions of the control problem, where the latter is based on local information. We showed that the offline algorithm can be used to pass/fail possible autonomous vehicle behaviors against prioritized driving rules. We also showed that our receding-horizon implementation of the online algorithm achieves better tracking. We presented multiple case studies for an autonomous vehicle with realistic dynamics and conflicting rules. Future work will focus on learning priority structures from data, improving the feasibility of the control problems using machine learning and optimization techniques, 
and refinement of the pass/fail procedure.

\bibliographystyle{IEEEtran}
\bibliography{Biblio}

\section*{APPENDIX}
\section{Rule definitions}\label{sec:app-rule-def}
Formal definitions for all the rules used 
in Sec. \ref{sec:case} can be found in \cite{Xiao2021ICCPS}. According to Def. \ref{def:rule}, 
each rule should be satisfied for all times. Here we include only $r_1$ for illustration. 

\begin{equation}
    \begin{aligned} \label{eqn:r_1}
    r_1: &\text{ Maintain clearance with pedestrians}\\
    &\text{Statement: } d_{min,fp}(\bm x, \bm x_i)\geq d_{1} + v(t)\eta_1,\forall i\in S_{ped}\\
     &\varrho_{r,i}(\bm x(t)) = \max(0, \frac{d_{1} + v(t)\eta_1 - d_{min,fp}(\bm x, \bm x_i)}{d_{1} + v_{max}\eta_1})^2,\\
     &\rho_{r,i}(\mathcal{X}) = \max_{t\in[0,T]} \varrho_{r,i}(\bm x(t)), \quad
     P_r = \sqrt{\frac{1}{n_{ped}}\sum_{i\in S_{ped}}\rho_{r,i}}.
    \end{aligned}
\end{equation}
where $d_{min,fp}:\mathbb{R}^{n+n_i}\rightarrow {\mathbb{R}}$
denotes the distance between footprints of ego and pedestrian $i$, and $d_{min,fp}(\cdot,\cdot) < 0$ denotes the footprint overlap.  The clearance threshold is given based on a fixed distance $d_1\geq 0$ and increases linearly by $\eta_1 > 0$ based on ego speed $v(t) \geq 0$ ($d_1$ and $\eta_1$ are determined empirically). $S_{ped}$ denotes the index set of all pedestrians, and $\bm x_i\in \mathbb{R}^{n_i}$ denotes the state of pedestrian $i$. $v_{max}$ is the maximum feasible speed of the vehicle and is used to define the normalization term in $\varrho_{r,i}$, which assigns a violation score (based on a L-$2$ norm) if formula is violated by $\bm x(t)$. $\rho_{r,i}$ defines the instance violation score as the most violating instant over $\mathcal{X}$. $P_r$ aggregates the instance violations, where $n_{ped}\in\mathbb{N}$ denotes the number of pedestrians.

\section{Optimal Disk Coverage}\label{sec:app-coverage}

To construct disks to fully cover the clearance regions, we need to find their number and radius. From Fig. \ref{fig:proof}, the lateral approximation error $\sigma >0$ is given by:
\begin{equation}\label{eq:sigma}
    \sigma = \mathfrak{r} - \frac{w + h_l(\bm x) + h_r(\bm x)}{2}.
\end{equation}
Since $\sigma$ for ego depends on its state $\bm x$ (speed-dependent), we consider the accumulated lateral approximation error for all possible $\bm x\in X$. This allows us to determine $z$ and $\mathfrak{r}$ such that the disks fully cover ego clearance region for all possible speeds in $\bm x$. Let $\bar h_i = \sup_{\bm x\in X}h_i(\bm x), \underline{h}_i = \inf_{\bm x\in X}h_i(\bm x), i\in\{f,b,l,r\}$. We can formulate the construction of the approximation disks as an optimization problem:

\begin{equation}\label{eqn:opcircle}
    \min_{z} z + \beta \int_{\underline{h}_f}^{\bar h_f}\int_{\underline{h}_b}^{\bar h_b}\int_{\underline{h}_l}^{\bar h_l}\int_{\underline{h}_r}^{\bar h_r}\sigma dh_f(\bm x)dh_b(\bm x)dh_l(\bm x)dh_r(\bm x)
\end{equation}
subject to $
    z \in \mathbb{N},
    $
where $\beta\geq 0$ is a trade-off between minimizing the number of the disks (so as to minimize the number of constraints considered with CBFs) and the coverage approximation error. The above problem is solved offline. 

A similar optimization is formulated for construction of disks for instances in $S_p$ (we remove the integrals due to speed-independence). Note that for the driving scenarios studied in this paper, we omit the longitudinal approximation errors in the front and back. The lateral approximation errors are considered in the disk formulation since they induce conservativeness in the lateral maneuvers of ego required for surpassing other instances (such as parked car, pedestrians, etc.), see Sec. \ref{sec:case}.

Let $(x_e,y_e) \in\mathbb{R}^2$ be the center of ego and $(x_i,y_i) \in\mathbb{R}^2$ be the center of instance $i\in S_p$. The center of disk $j$ for ego $(x_{e,j}, y_{e,j}), j\in \{1,\dots,z\}$ is determined by:
\begin{equation} \label{eqn:center}
\begin{aligned}
    x_{e,j} = x_e + \cos\theta_e(-\frac{l}{2} - h_b(\bm x) + \frac{l+h_f(\bm x) + h_b(\bm x)}{2z}(2j - 1))\\
    y_{e,j} = y_e + \sin\theta_e(-\frac{l}{2} - h_b(\bm x) + \frac{l+h_f(\bm x) + h_b(\bm x)}{2z}(2j - 1))
\end{aligned}
\end{equation}
where $j\in\{1,\dots, z\}$ and $\theta_e\in\mathbb{R}$ denotes the heading angle of ego. The center of disk $k$ for instance $i \in S_p$ denoted by $(x_{i,k}, y_{i,k}), k\in \{1,\dots,z_i\}$, can be defined similarly. 

\begin{theorem} \label{thm:cover}
If the clearance regions of ego and instance $i\in S_p$ are covered by the disks constructed by solving (\ref{eqn:opcircle}), then the clearance regions of ego and instance $i$ do not overlap if (\ref{eqn:rule_cons}) is satisfied. 
\end{theorem}
\begin{proof}

Let $z$ and $z_i$ be the disks with minimum radius $\mathfrak{r}$ and $\mathfrak{r}_i$ from (\ref{eqn:radius}) associated with the clearance regions of ego and instance $i\in S_p$. The constraints in (\ref{eqn:rule_cons}) guarantee that there is no overlap of the disks between vehicle $i\in S_p$ and instance $j\in S_p$. Since the clearance regions are fully covered by these disks, we conclude that the clearance regions do not overlap.
\end{proof}

\section{Software tool and running times}
\label{sec:tool}

We implemented the computational procedure described in this paper in Matlab. The tool allows to load a map represented by a .json file and place vehicles and pedestrians on it. It provides an interface to generate reference/candidate trajectories and it implements our proposed optimal control and P/F frameworks; the $quadprog$ optimizer was used to solve the QPs, the $fmincon$ solver was used to solve the NLPs, and $ode45$ was used to integrate the vehicle dynamics (\ref{eqn:vehicle}). All the computation was performed on a Intel(R) Core(TM) i7-8700 CPU @
3.2GHz$\times 2$. For offline control, it took about 0.04s for each QP, and there were about 400 QPs in each iteration, and each case study we considered in this paper took two or three iterations to find an optimal/feasible solution. For online control, it took about 0.90s for each step (including the NLP (\ref{eqn:mpc}) and the QP (\ref{eqn:cost_online})) to find an optimal control. 
The parameters used in Sec. \ref{sec:case} can be found in \cite{Xiao2021ICCPS}.
%

$\vspace{-2cm}$
\begin{IEEEbiography}[{\includegraphics[width=1in,height=1.25in,clip,keepaspectratio]{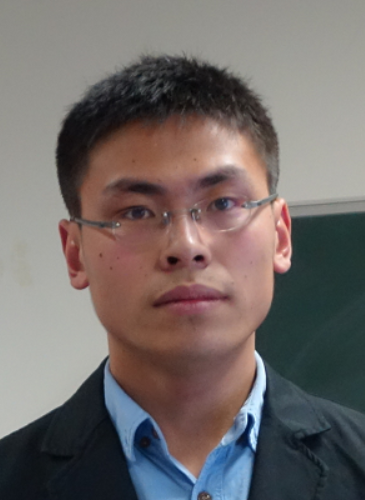}}]{Wei Xiao (S'19)}
received a B.Sc. degree from the University of Science and Technology Beijing, China in 2013 and a M.Sc. degree from the Chinese Academy of Sciences (Institute of Automation), China in 2016. He is currently working toward his Ph.D. degree in Systems Engineering at Boston University, Brookline, MA, USA.

His research interests include control theory, formal methods, and machine learning, with particular emphasis on robotics and traffic control. He received an Outstanding Student Paper Award at the 2020 IEEE Conference on Decision and Control.
\end{IEEEbiography}
$\vspace{-2cm}$
\begin{IEEEbiography}
[{\includegraphics[width=1in,height=1.25in,keepaspectratio]{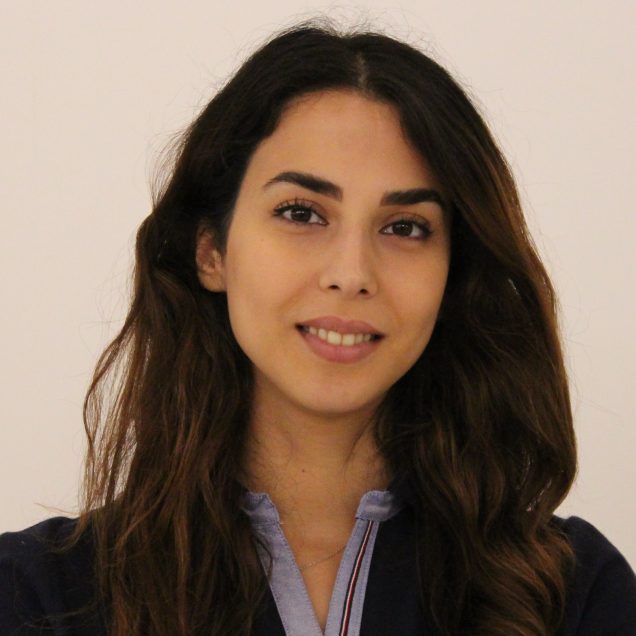}}]{Noushin Mehdipour}
	received a M.Sc. from Amirkabir University of Technology, Tehran, Iran in 2015, and a Ph.D. in Systems Engineering from Boston University, Massachusetts, USA in 2021. Her research interests include formal methods, control theory, machine learning and optimization methods. She is currently working as a research scientist in the Rulebooks team at Motional where she is focused on formal behavior specification, evaluation and control of autonomous vehicles. 
\end{IEEEbiography}
$\vspace{-2cm}$
\begin{IEEEbiography}[{\includegraphics[width=1in,height=1.25in,clip, keepaspectratio]{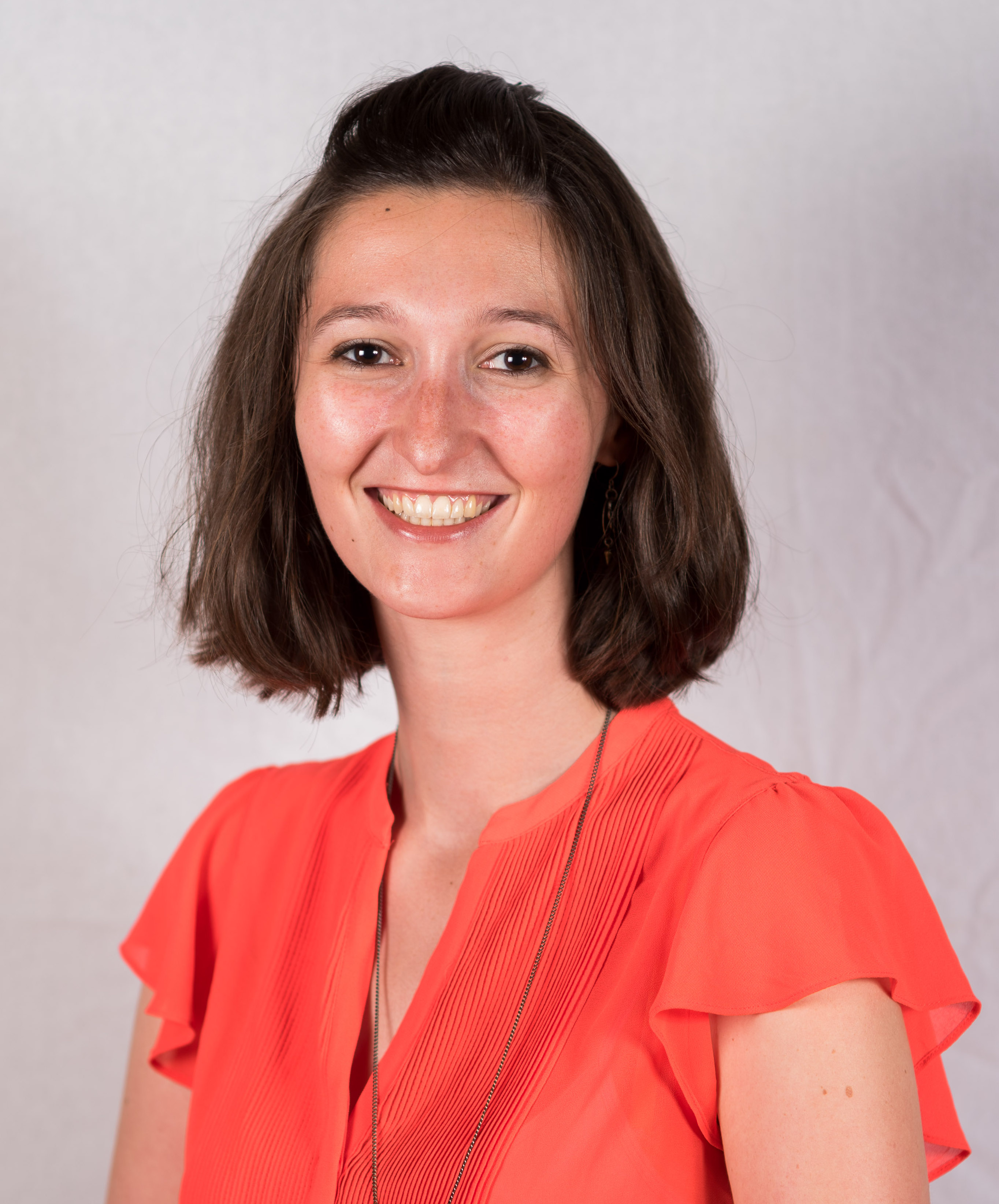}}]{Anne Collin}
received a M.Sc. from Ecole Nationale des Ponts et Chaussees, France in 2016, a M.Sc. in Technology and Policy and a Ph.D. in Aeronautics and Astronautics from the Massachusetts Institute of Technology, Massachusetts, USA in 2016 and 2019. She is a Senior Research Scientist in the Rulebooks team at Motional. Her research interests include complex systems modeling and architecture, optimization, and performance quantification of Artifical Intelligence Systems. Her work received a Best Student Paper Award at the 2018 IEEE ICVES conference.
\end{IEEEbiography}
$\vspace{-2cm}$
\begin{IEEEbiography}[{\includegraphics[width=1in,height=1.25in,clip, keepaspectratio]{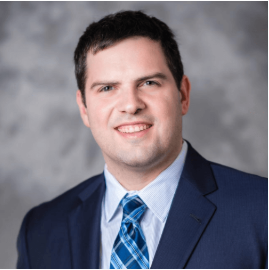}}]{Amitai Bin-Nun}
received a B.S. from Yeshiva University, New York, USA in 2006, an M.A.St from the University of Cambridge, Cambridge, UK in 2007, and a M.A. and Ph.D. from the University of Pennsylvania, Philadelphia, USA, in 2008 and 2010. He is a Senior Research Scientist at Motional, based in Boston, USA. His interests lie at the intersection of artificial intelligence, transportation, and public policy. Amitai's previous roles include service as vice president of a non-profit that developed public policies for autonomous vehicles, and as a science policy fellow at the United States Department of Energy and for United States Senator Chris Coons.
\end{IEEEbiography}
$\vspace{-2cm}$
\begin{IEEEbiography}[{\includegraphics[width=1in,height=1.25in,clip, keepaspectratio]{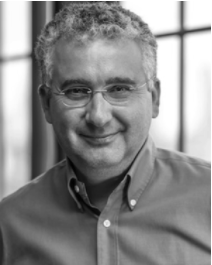}}]{Emilio Frazzoli}
 is currently a Professor of dynamic systems and control with ETH Zürich, Zürich, Switzerland, and the Co-Founder of nuTonomy Inc. He received the Laurea degree in aerospace engineering from the Sapienza University of Rome, Rome, Italy, in 1994, and the Ph.D. degree from the Department of Aeronautics and Astronautics, Massachusetts Institute of Technology, Cambridge, MA, USA, in 2001.,Before joining ETH Zürich in 2016, he held faculty positions at the University of Illinois at Urbana–Champaign, the University of California at Los Angeles, Los Angeles, CA, USA, and the Massachusetts Institute of Technology. His current research interests focus primarily on autonomous vehicles, mobile robotics, and transportation systems.,Dr. Frazzoli was a recipient of the NSF CAREER Award in 2002, the IEEE George S. Axelby Award in 2015, and the IEEE Kiyo Tomiyasu Award in 2017.
\end{IEEEbiography}
$\vspace{-2cm}$
\begin{IEEEbiography}[{\includegraphics[width=1in,height=1.25in,clip, keepaspectratio]{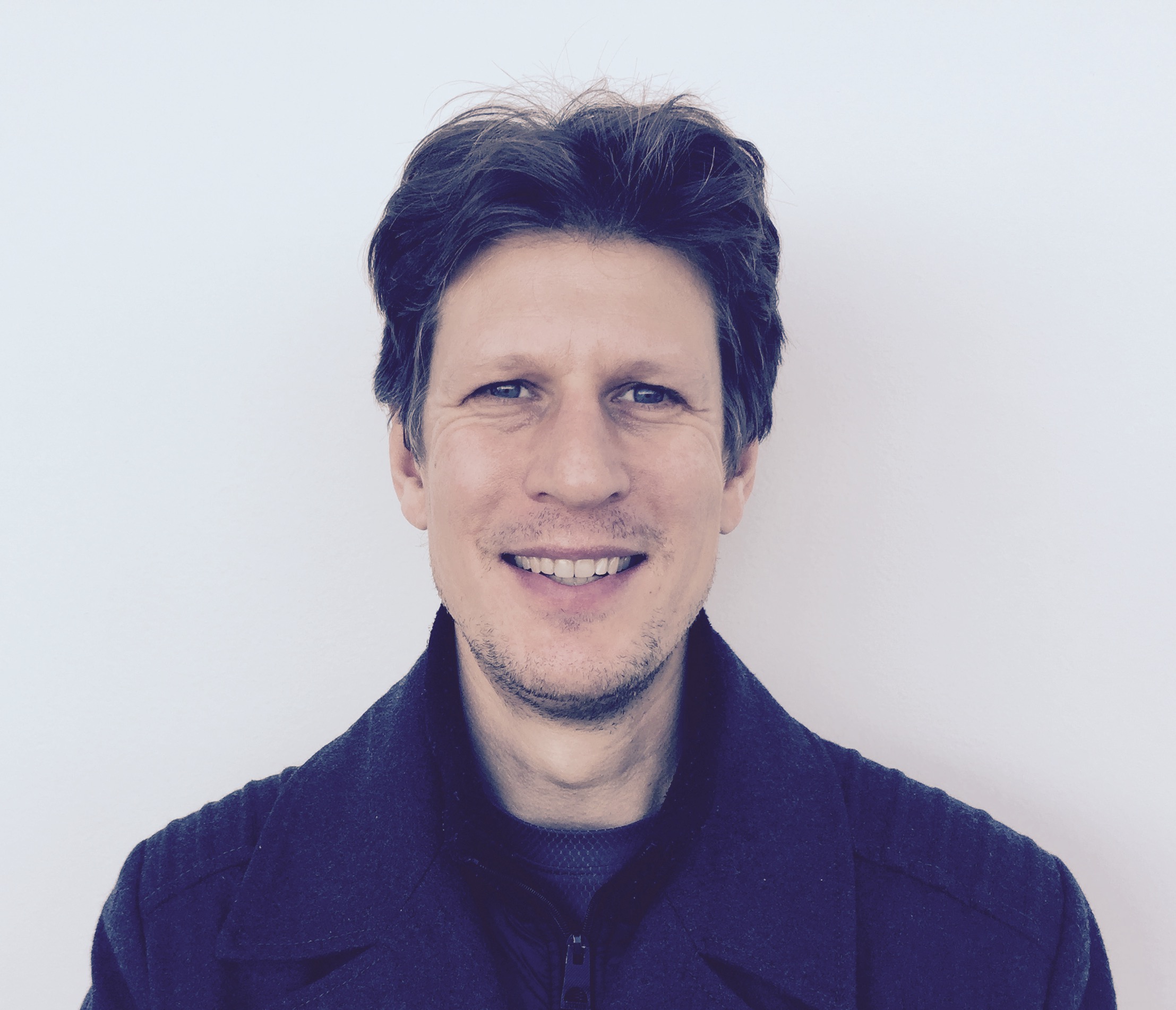}}]{Radboud Duintjer Tebbens}
received a M.Sc. and Ph.D. in Applied Mathematics from Delft University of Technology in the Netherlands. A former Professor of Risk and Decision Analysis, Rad's work focuses on decision making in the context of complex systems involving non-linear dynamics and uncertainty. As a Principal Data Scientist at Motional, he currently leads the Rulebooks team responsible for formal driving behavior specification and assessment and previously helped define the company's autonomous vehicle validation strategy. Prior to joining Motional, as Vice President of non-profit Kid Risk, Inc. he won the INFORMS' Franz Edelman award and the System Dynamics Society's Jay Forrester award for his work  on integrated mathematical models to support the Global Polio Eradication Initiative.
\end{IEEEbiography}
$\vspace{-2cm}$
\begin{IEEEbiography}[{\includegraphics[width=1in,height=1.25in,clip,keepaspectratio]{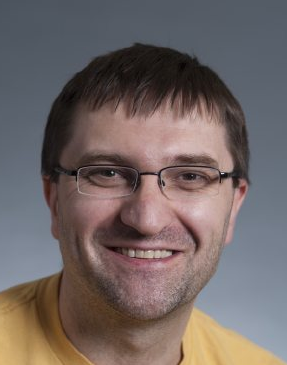}}]{Calin Belta (F'17)}
received B.Sc. and M.Sc, degrees from the Technical University of Iasi, Romania in 1995 and 1997 and M.Sc. and Ph.D. degrees from the University of Pennsylvania, Philadelphia, USA in 2001 and 2003. He is a Professor in the Department of Mechanical Engineering at Boston University, where he is the Director of the BU Robotics Lab. His research focuses on dynamics and control theory, with particular emphasis
on hybrid and cyber-physical systems, formal synthesis and verification, with applications to robotics, autonomous driving, and systems biology. Awards include the Air Force Office of Scientific Research Young Investigator Award, the National Science Foundation CAREER Award, and the 2017 IEEE TCNS Outstanding Paper Award. He is a fellow of IEEE and a distinguished lecturer of the IEEE Control System Society. 
\end{IEEEbiography}





\end{document}